\documentclass[10pt]{article} 
\usepackage[preprint]{tmlr}

\usepackage{amsmath,amsfonts,bm}

\def\eqref#1{equation~\ref{#1}}

\def\1{\bm{1}}

\DeclareMathAlphabet{\mathsfit}{\encodingdefault}{\sfdefault}{m}{sl}
\SetMathAlphabet{\mathsfit}{bold}{\encodingdefault}{\sfdefault}{bx}{n}

\usepackage{hyperref}
\usepackage{url}

\usepackage{algorithm}
\usepackage{algpseudocode} 
\usepackage{multirow}
\usepackage{booktabs}    
\usepackage{caption}     
\usepackage{graphicx}    
\usepackage{array}       
\usepackage{float}       
\usepackage{enumitem}
 \usepackage{multirow}
\usepackage{array}
\usepackage{tabularx}
\usepackage{booktabs}
\usepackage{dsfont}
\usepackage{amsfonts}
\newtheorem{definition}{Definition}[section]
\newtheorem{theorem}{Theorem}[section]

\newtheorem{lemma}[theorem]{Lemma}
\newtheorem{remark}{Remark}

\newtheorem{proof}[theorem]{Proof}
\newtheorem{note}{Note}[section]

\title{APFEx: Adaptive Pareto Front Explorer for Intersectional Fairness}

\author{\name Priyobrata Mondal \email priyobrata1996\_r@isical.ac.in \\
      \addr Electronics and Communication Sciences Unit \\
        Indian Statistical Institute, Kolkata
      \AND
      \name Faizanuddin Ansari \email faizanuddin\_r@isical.ac.in \\
      \addr Electronics and Communication Sciences Unit \\
        Indian Statistical Institute, Kolkata
      \AND
      \name Swagatam Das \email swagatam.das@isical.ac.in\\
      \addr Electronics and Communication Sciences Unit \\
        Indian Statistical Institute, Kolkata
  }

\def\month{MM}  
\def\year{YYYY} 
\def\openreview{\url{https://openreview.net/forum?id=XXXX}} 

\begin{document}
\maketitle

\begin{abstract}
Ensuring fairness in machine learning models is critical, especially when biases compound across intersecting protected attributes like race, gender, and age. While existing methods address fairness for single attributes, they fail to capture the nuanced, multiplicative biases faced by intersectional subgroups. We introduce \emph{Adaptive Pareto Front Explorer (APFEx)}, the first framework to explicitly model intersectional fairness as a joint optimization problem over the Cartesian product of sensitive attributes. APFEx combines three key innovations: (1) an adaptive multi-objective optimizer that dynamically switches between Pareto cone projection, gradient weighting, and exploration strategies to navigate fairness-accuracy trade-offs; (2) differentiable intersectional fairness metrics enabling gradient-based optimization of non-smooth subgroup disparities; and (3) theoretical guarantees of convergence to Pareto-optimal solutions. Experiments on four real-world datasets demonstrate APFEx’s superiority, reducing fairness violations while maintaining competitive accuracy. Our work bridges a critical gap in fair ML, providing a scalable, model-agnostic solution for intersectional fairness.
\end{abstract}

\section{Introduction}
\textbf{Beyond Accuracy—Fairness at the Intersections.} When a healthcare algorithm used across U.S. hospitals was found to systematically prioritize healthier white patients over sicker Black patients for lifesaving care, it exposed a harsh truth: even accurate machine learning models can fail basic tests of fairness \citep{barocas2023fairness}. Such cases reveal how biases embedded in data—whether based on race, gender, or income—can perpetuate discrimination when deployed at scale. But the problem runs deeper. Consider the COMPAS recidivism tool: while it appeared marginally fair when examining race or gender separately, Black women were disproportionately misclassified as high-risk, facing compounded bias at the intersection of these identities \citep{angwin2022machine}. This phenomenon—where discrimination multiplies across overlapping attributes like race, gender, and age—is the crux of \emph{intersectional fairness}, a challenge single-attribute approaches cannot solve \citep{chen2024fairness}. As AI systems increasingly govern access to loans, healthcare, and employment, ensuring fairness for \emph{all} individuals, not just broad demographic groups, becomes both a technical necessity and an ethical imperative.

\textbf{The Shortcomings of Current Fairness Techniques. }To address fairness challenges, the ML community has introduced various criteria—such as demographic parity, equalized odds, and equality of opportunity—each grounded in distinct normative principles of equitable treatment \citep{barocas2023fairness, zafar2017fairness}. These have given rise to three main approaches: (1) \emph{pre-processing} methods that adjust input data distributions to reduce bias \citep{kamiran2012data, calmon2017optimized, sattigeri2019fairness}, (2) \emph{in-processing} techniques that incorporate fairness constraints directly into model training \citep{zafar2017fairness, celis2019classification, lohaus2020too}, and (3) \emph{post-processing} strategies that modify model outputs to achieve fairness \citep{hardt2016equality, chzhen2019leveraging, chierichetti2017fair}. Despite significant progress, these methods predominantly adopt a \emph{single-attribute} perspective, treating sensitive features (e.g., race, gender) in isolation and enforcing fairness constraints independently using additive penalties. This reductionist framing fails in \emph{multi-attribute} contexts, where bias arises from the intersection of protected dimensions. As shown by \citep{angwin2022machine}, models that appear fair when evaluated separately on race and gender can still severely discriminate against intersectional groups, such as Black women. The core issue is that single-attribute fairness lacks the representational capacity to capture complex, higher-order discrimination patterns—precisely where the most harmful biases often occur.

\textbf{The challenge of intersectional fairness is inherently multi-objective.} Algorithmic fairness requires optimizing predictive accuracy alongside fairness metrics that often conflict \citep{Martinez2020}. Recent work has framed this as a multi-objective problem: \citet{valdivia2021fair} use evolutionary algorithms; \citet{milojkovic2019multi} and \citet{cotter2019two} adopt multi-gradient and game-theoretic methods—but all treat sensitive attributes independently and rely on scalarized losses, obscuring intersectional effects. \citet{padh2021addressing} proposed a model-agnostic, differentiable framework, yet combined fairness penalties additively. \citet{Martinez2020} introduced static Pareto optimization without dynamic trade-off navigation, and \citet{nagpal2024multi} jointly optimized group and individual fairness using ROC, but without compositional modeling. Crucially, no existing work captures the non-additive structure of intersectional bias or offers dynamic control over the accuracy–fairness Pareto front.

While these efforts mark meaningful progress, they leave a critical research gap—they stop short of capturing the compositional complexity of intersectional bias. The inability to model fairness over the full joint distribution of sensitive features leaves existing methods blind to the most vulnerable subgroups. \emph{To our knowledge, this work is the first to explicitly model and optimize fairness over the Cartesian product of protected attributes—directly addressing intersectional bias as a compositional, multi-objective problem.}

\textbf{Our work bridges this gap through four key advances.} First, we reformulate intersectional fairness as a \emph{joint optimization} over the Cartesian product of sensitive attributes, capturing compounded biases via subgroup-wise constraints (e.g., ensuring parity for Black women directly, not just marginally for race or gender). Second, we introduce \textit{APFEx (\underline{A}daptive \underline{P}areto \underline{F}ront \underline{Ex}plorer)}, an adaptive multi-objective optimizer that dynamically switches between: (1) \emph{Pareto cone projection} for aligned objectives, (2) \emph{gradient weighting} for imbalanced improvements, and (3) \emph{calibrated exploration} to escape local optima—theoretically guaranteed to converge to Pareto-stationary solutions (\autoref{thm:pareto_convergence}). Third, we develop \emph{differentiable intersectional metrics} that enable gradient-based optimization of otherwise non-smooth subgroup fairness (e.g., hyperbolic tangent relaxations for demographic parity). Fourth, we also introduced a generalized fairness loss that handles multi-class classification with intersectional fairness (\autoref{multiclass-loss}).

\section{Related Work}

Fairness in machine learning has traditionally been studied under three broad methodological families: pre-processing, in-processing, and post-processing interventions. While these approaches have yielded important insights, they remain limited in their ability to address \emph{intersectional fairness}, i.e., fairness across subgroups defined by joint combinations of multiple sensitive attributes. Below we review these categories before positioning our work, which directly tackles intersectional fairness through a multi-objective optimization framework with adaptive Pareto front exploration.

\noindent \textbf{Pre-Processing and Post-Processing Approaches.}  
Pre-processing methods aim to mitigate bias by modifying the input distribution prior to training. Early work by \citet{kamiran2012data} adjusted class labels or sample weights to suppress discriminatory patterns. Later contributions introduced probabilistic and generative data transformations for fairness-preserving representations \citep{calmon2017optimized, sattigeri2019fairness}. While these methods are model-agnostic, they often misalign with a learner’s inductive bias, limiting their generalizability.  

Post-processing methods, in contrast, operate on model outputs. \citet{hardt2016equality} proposed threshold calibration for equalized opportunity, while \citet{chzhen2019leveraging} refined this idea through semi-supervised calibration. These approaches are lightweight and modular, but their corrective power is bounded by the representational capacity of the underlying model, making them ineffective when bias is deeply encoded in learned features. Importantly, both pre- and post-processing strategies typically focus on a single sensitive attribute, offering no principled mechanism for handling intersectional subgroups.

\noindent \textbf{In-Processing and Relaxation-Based Methods.}  
In-processing approaches directly integrate fairness constraints into model training, often by formulating the task as a constrained optimization problem. A common strategy replaces non-differentiable fairness indicators with differentiable relaxations. \citet{zafar2017fairness} proposed covariance-based relaxations for demographic parity and equalized opportunity, yielding convex or convex–concave formulations. \citet{donini2018empirical} adopted an empirical risk minimization approach under fairness constraints, and \citet{celis2019classification} introduced a meta-algorithm that can optimize classification subject to multiple fairness metrics using estimated conditional probabilities.

However, relaxations can sometimes be too permissive: \citet{lohaus2020too} demonstrated that models may satisfy relaxed constraints while violating fairness in the true sense. Their search-based regularization approach offers stronger guarantees but incurs substantial computational cost. To improve scalability, \citet{padh2021addressing} framed fairness as a multi-objective problem and introduced a differentiable relaxation using hyperbolic tangent functions. This enabled simultaneous optimization of multiple fairness definitions across sensitive attributes. Yet, their method still treated attributes separately and failed to capture their joint interactions. Similarly, \citet{nagpal2024multi} proposed a reject-option classification framework optimized with NSGA-II and particle swarm algorithms, achieving improved accuracy–fairness trade-offs but without explicit modeling of fairness over the joint distribution of protected attributes.

Overall, in-processing methods have provided more principled mechanisms for balancing fairness and accuracy, but most remain tied to single-attribute settings or require modifications for each fairness metric.
\noindent \textbf{Multi-Objective and Intersectional Fairness.}  
A growing line of research has considered fairness explicitly as a multi-objective problem. \citet{valdivia2021fair} applied evolutionary algorithms to optimize both fairness metrics and predictive accuracy, while \citet{milojkovic2019multi} and \citet{cotter2019two} developed multi-gradient descent and game-theoretic formulations for non-convex fairness–accuracy optimization. Despite these advances, most existing methods either treat sensitive attributes separately or rely on scalarized losses that obscure the performance of intersectional subgroups.

Several works specifically address intersectional fairness. For instance, \citet{foulds2018intersectional} propose a Bayesian probabilistic modeling approach that enables data-efficient estimation of fairness across multiple protected attributes by handling sparsity and statistical uncertainty. However, their focus lies on measurement and estimation, rather than on optimizing classifier decisions under intersectional fairness constraints. Another relevant contribution is by \citet{mickel2023importance},
who develop both auditing and post-processing techniques for intersectional fairness in classification. Their methods extend single-attribute fairness metrics to the intersectional setting and offer robust, model-agnostic bias mitigation. Yet, these are post-hoc corrections, lacking an integrated optimization framework.

\noindent \textbf{Our Contribution.}  
Based on the above category that we have discussed our method lies in the category of in-processing and multi-objective optimization. Within our in-processing framework, we integrate fairness constraints associated with sensitive attributes into the model training phase, typically by framing the task as a constrained optimization problem. To the best of our knowledge, no prior work explicitly formulates fairness over the \emph{joint sensitive attribute space} within a unified, differentiable, and multi-objective optimization framework. We bridge this gap by constructing composite subgroup labels from sensitive attribute tuples and jointly optimizing \emph{performance loss} and \emph{fairness loss} using an adaptive Pareto front explorer. Unlike scalarization-based or post-hoc methods, our approach preserves the trade-off structure explicitly, enabling more nuanced control over the fairness–accuracy frontier. Moreover, it is model-agnostic and readily integratable into standard supervised learning pipelines. This makes it the first method to systematically address intersectional fairness via principled multi-objective optimization, producing Pareto-optimal classifiers that respect fairness across all defined subgroups simultaneously.

\section{Methodology}

This section develops a principled optimization framework for jointly minimising prediction error and intersectional fairness violations. We begin by formally stating the problem and explaining why the particular mathematical objects we introduce are needed. We then recall key preliminaries from multi-objective optimization (MOO) that underpin our method, present the \emph{Adaptive Pareto Front Explorer (APFEx)} framework and its core components, and finish with the concrete differentiable losses we use to encode fairness constraints.
\subsection{Problem Definition}
We consider a multi-objective optimization problem where we seek to balance predictive accuracy with fairness constraints across intersectional subgroups. Let $\mathcal{D} = \{(x_i, a_i, y_i)\}_{i=1}^n$ denote our dataset, where $x_i \in \mathbb{R}^d$ represents the feature vector, $a_i = (a_i^{(1)}, \ldots, a_i^{(K-1)}) \in \mathcal{A}_1 \times \cdots \times \mathcal{A}_{K-1}$ represents the vector of sensitive attributes having nominal values with feature space as $\mathcal{A}_i \in \{ -1, 1\}$, and $y_i \in \{-1, 1\}$ is the binary label. Consider a probability distribution of the training samples $\mathbb{S} = (\mathcal{X}, \mathcal{A}, \mathcal{Y})$ be ${\mathbb{P}}^{\mathbb{S}}$, where each $\{x_i,a_{i1},\cdots,a_{i(k-1)},y_i\}$ is sampled \textit{i.i.d} from ${\mathbb{P}}^{\mathbb{S}}$.

The intersectional fairness problem requires us to consider fairness across the Cartesian product of sensitive attributes. For $K-1$ sensitive attributes, we define the set of intersectional groups as:
$$\mathcal{G} = \mathcal{A}_1 \times \mathcal{A}_2 \times \cdots \times \mathcal{A}_{K-1}$$

Each intersectional group $g \in \mathcal{G}$ represents a unique combination of sensitive attribute values, capturing the compounded bias that individuals at intersections may face. The learning problem we address is inherently multi-objective: we must trade off a predictive objective and fairness objectives defined across the intersectional groups. In the following definition~\ref{def1}, we have formally defined this.


\begin{definition}[Multi-Objective Intersectional Fairness Problem]
\label{def1}
Given a parametric model $f_\theta: \mathbb{R}^d \to \mathbb{R}$ with parameters $\theta \in \Theta \subseteq \mathbb{R}^p$, we seek to solve:
\begin{align}
\min_{\theta \in \Theta} \mathbf{L}(\theta) = \begin{bmatrix} L_{\text{task}}(\theta) \\ L_{\text{fairness}}(\theta) \end{bmatrix}
\end{align}
where $L_{\text{task}}(\theta) = \mathbb{E}_{(x,y) \sim \mathcal{D}}[\ell(f_\theta(x), y)]$ measures predictive performance, and $L_{\text{fairness}}(\theta)$ captures fairness violations across intersectional groups. 
\end{definition}
\begin{note}
We have K objective functions, one for the task prediction and $(K-1)$ for $(K-1)$ sensitive attributes.
\end{note}
\subsection{Preliminaries}
Before introducing APFEx, we briefly recall MOO concepts that motivate our method for fairness-aware training.
\subsubsection{Multi-objective optimization (MOO)}
\emph{Multi-objective optimization (MOO)} addresses problems where multiple, potentially conflicting objectives must be simultaneously optimized. Unlike single-objective optimization, the solution to a multi-objective problem is typically not a single point but rather a set of points representing various trade-offs among the objectives.
We define the Multi-objective optimization as: 
\begin{equation}
    \underset{x_i\in \mathcal{X}}{min}\text{ } \mathbb{F}(x_i) = [ f_1(x_i),f_2(x_i),\cdots,f_K(x_i) ]^T
\end{equation}
where $\mathbf{x} = [x_1, x_2, \ldots, x_n]^T$ is the decision vector, $\mathcal{X}\subseteq \mathbb{R}^d$ is the feasible region, and $\mathbb{F}: \mathcal{X} \rightarrow \mathbb{R}^K$ is the objective function vector with $K$ individual objective functions $f: \mathcal{X} \rightarrow \mathbb{R}$, $f_i(x_j)$ is the loss function for the $i^{th}$ objective. 

\subsubsection{Pareto dominance and Pareto front}
At the core of multi-objective optimization is Pareto optimality, which formalizes the notion of trade-offs between competing objectives. To begin formalizing this framework, we first define a concept that allows us to compare solutions based on their performance across multiple objectives.
\begin{definition}[\textbf{Pareto Dominance}]
A solution $f(x^1)$ is said to dominate another solution $\mathbf{x}^{2}$ (denoted as $\mathbf{x}^{1} \prec \mathbf{x}^{2}$) iff:
\begin{itemize}[noitemsep,topsep=0pt]
    \item[1.] $\forall i \in \{1,2,\ldots,K\}: f_i(x^{1}) \leq f_i(x^2)$, and
    \item[2.] $\exists j \in \{1,2,\ldots,K\}: f_j(x^1) < f_j(x^2)$
\end{itemize}
\end{definition}
This notion of dominance allows us to identify solutions that are better than others in at least one objective and no worse in all others. Building on this, we can define the concept of Pareto optimality, which characterizes solutions that are not dominated by any other.
\begin{definition}[\textbf{Pareto Optimality}]: A solution $x^*$ is Pareto optimal if there exists no other solution $x \in \mathcal{X}$ such that $f(x)\prec x^{*}$.
\end{definition}
The collection of all Pareto optimal solutions forms a special set in the objective space, which we define next.
\begin{definition}[\textbf{Pareto Front}]: The set of all Pareto optimal solutions in the objective space, denoted as $PF = \{F(x) \mid x \in PS\}$, where $PS$ is the set of all Pareto optimal solutions in the decision space.
\end{definition}



These concepts are central to fairness-aware learning because they let us characterise the set of classifiers that achieve different accuracy–fairness trade-offs rather than forcing a single scalarised compromise. To solve a multi-objective optimization problem, we aim to identify or approximate the Pareto front. Consider $K$ objective functions $f_1(\mathbf{w}), f_2(\mathbf{w}), \ldots, f_K(\mathbf{w})$ with a parameter vector $\mathbf{w} \in \mathbb{R}^d$. A fundamental challenge is to identify a direction in parameter space that reduces all objective functions simultaneously.
This leads us to the concept of Pareto stationarity, which provides a necessary condition for a solution to be Pareto optimal in continuous spaces.
\begin{definition}[\textbf{Pareto Stationary \cite{Fliege2000}}]
A point $\mathbf{x} \in \Omega$ is Pareto stationary if there does not exist a feasible direction $\mathbf{d} \in \mathbb{R}^n$ such that:
\begin{equation}
    \nabla f_i(x)^Td < 0, \forall i\in \{1,2,\cdots, K\}.
\end{equation}    
\end{definition}
The following Theorem~\ref{thm1} gives a convenient geometric characterisation that we exploit algorithmically: Pareto stationarity holds iff the zero vector lies in the convex hull of the gradients.

\begin{theorem}
\label{thm1}
A point $\mathbf{x} \in \Omega$ is Pareto stationary\cite{Fliege2000} iff $\mathbf{0} \in \text{conv}({\nabla f_1(\mathbf{x}), \nabla f_2(\mathbf{x}), \ldots, \nabla f_K(\mathbf{x})})$, where $\text{conv}$ denotes the convex hull.
\end{theorem}
\begin{proof}
    The proof is provided in Appendix~\ref{thm_hull_conv}.
\end{proof}
Building on Pareto optimality and stationarity, we propose the \emph{\underline{A}daptive \underline{P}areto \underline{F}ront \underline{Ex}plorer (APFEx)}—a multi-objective framework that dynamically explores the Pareto front and manages trade-offs between competing objectives. APFEx extends gradient-based methods by combining multiple descent strategies with an adaptive selection mechanism, and explicitly formulates intersectional fairness as an optimization objective to jointly minimize prediction error and subgroup-level fairness violations.



\subsection{Proposed Methodology}
Building on the above preliminaries, we propose the \emph{Adaptive Pareto Front Explorer (APFEx)}: a dynamic multi-objective optimisation framework tailored to balance predictive performance and intersectional fairness. APFEx maintains gradient-scale invariance, adaptively chooses among complementary descent strategies, and injects controlled exploration when optimisation stalls. Below we present the components and their motivations, followed by convergence results. Following our discussions on Gradient Normalization~\ref{grad_norm_scale} and the APFEx framework~\ref{apfex_framework},  
we next reviewed the Existing Loss Function using Differential Metric Approximation in  
Subsection~\ref{existing_loss_diff}.  
Furthermore, in Section~\ref{multiclass-loss}, we presented the Fairness-Constrained Loss Function,  
which is formulated through a Differentiable Fairness Metric Approximation.
\subsubsection{Gradient Normalization and Scale Invariance}
\label{grad_norm_scale}
Different objectives (e.g., cross-entropy and group-level disparity) can vary greatly in magnitude, so directly aggregating their gradients may allow large-scale losses to dominate. To ensure a fair comparison, we estimate an empirical scale for each objective and normalize the gradients accordingly.

\begin{definition}[Empirical Loss Scale Estimation]
For each objective $k \in \{1, \ldots, K\}$, we estimate the empirical maximum:
$$\hat{L}_k^{\max} = \max_{(x,y) \in \mathcal{D}_{\text{train}}} \ell_k(f_\theta(x), y) + \epsilon$$
where $\epsilon > 0$ is a small regularization constant to prevent division by zero.
\end{definition}

The normalized gradient for objective $k$ is then computed as:
$$\tilde{\nabla}_\theta L_k(\theta) = \frac{\nabla_\theta L_k(\theta)}{\hat{L}_k^{\max}}$$

Normalisation preserves the geometry of each objective while preventing numerical dominance by any single loss. This step is used before forming aggregated descent vectors in APFEx.

\subsubsection{APFEx Framework}
\label{apfex_framework}
To optimize the trade-off between predictive performance and intersectional fairness, we propose APFEx (Figure~\ref{fig:apex_mot}): a dynamic, multi-objective framework. We first present its core descent strategies (PCP, AW, PSS), followed by the adaptive mechanism for strategy selection, constraint handling, and normalization. We then unify these components into a differentiable loss with theoretical convergence guarantees, demonstrating APFEx’s ability to achieve Pareto-optimal solutions while preserving intersectional fairness. In the subsequent Sections ~\hyperref[sub_des]{3.3.2.1} and ~\hyperref[sub_con]{3.3.2.2}, we described the different Strategies of the APFEx framework and the convergence analysis of our APFEx framework, respectively.

%
\begin{figure}[!ht]
    \centering
    \includegraphics[width=0.6\linewidth]{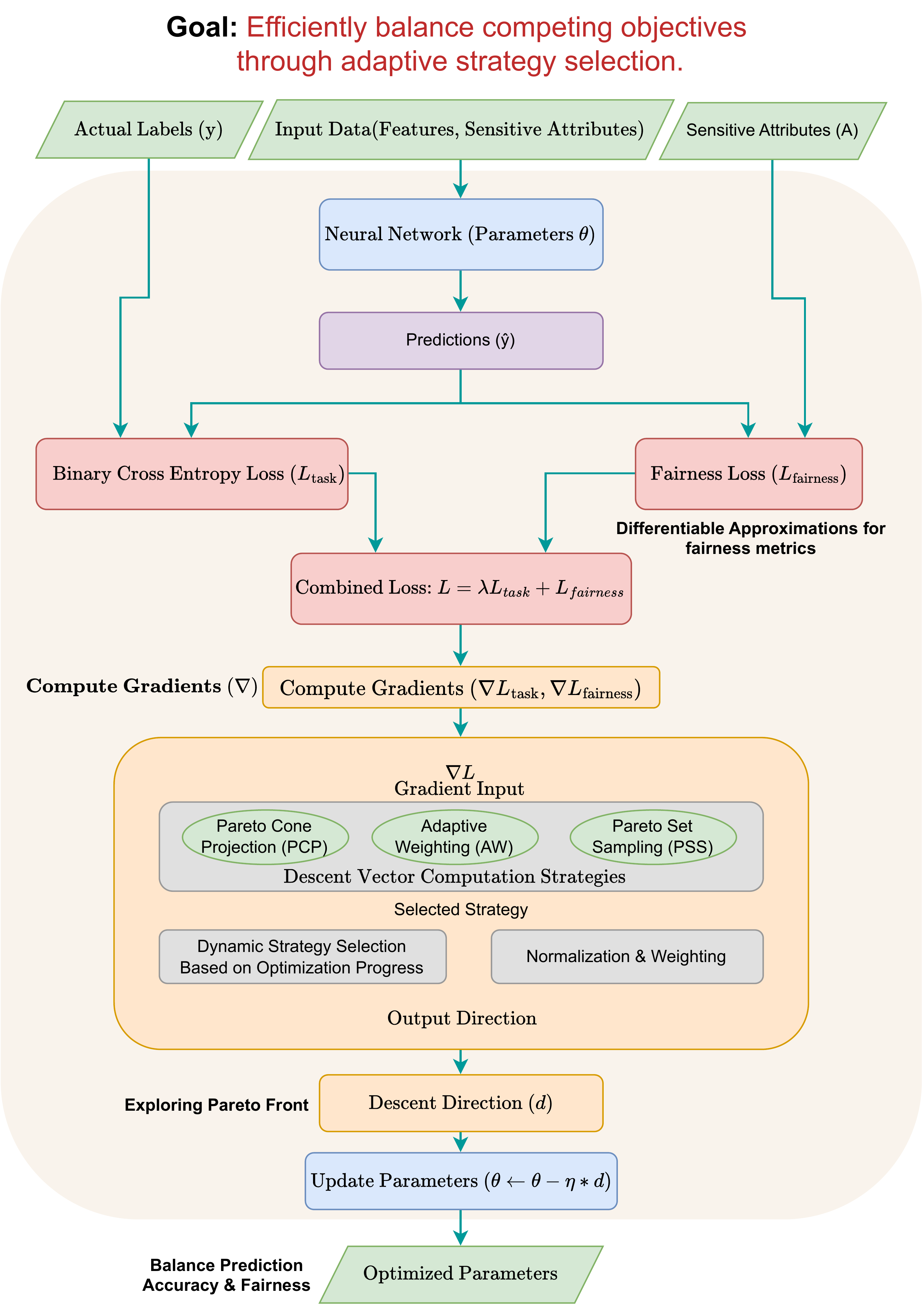}
    \caption{Figure illustrates the workflow of Adaptive Pareto Front Explorer (APFEx). }
    \label{fig:apex_mot}
\end{figure}
\paragraph{3.3.2.1 Descent Vector Computation Strategies}
\label{sub_des}
The core innovation of APFEx lies in its adaptive selection among three complementary descent strategies, each optimized for different optimization landscape characteristics. In Subsections \hyperref[pcp_stretegy]{(i)}, \hyperref[aw_stretegy]{(ii)}, and \hyperref[pss_stretegy]{(iii)}, we have discussed the PCP, AW, and PSS strategy, respectively, and we have presented the detailed algorithm in Algorithm~ref{alg:apfex}.

\noindent
\subparagraph{(i) Pareto Cone Projection (PCP)} 
\label{pcp_stretegy}
This strategy projects gradients onto the Pareto cone~\citep{Fliege2000} to ensure descent for all objectives simultaneously, thereby ensuring Pareto efficiency. We have defined the Pareto cone projection strategy in definition~\ref{defpcp}.





\begin{definition}[Pareto Cone~\citep{Fliege2000}]
\label{defpcp}
Given gradients $\{\nabla L_k(\theta)\}_{k=1}^K$, the Pareto cone $\mathcal{C}(\theta)$ is defined as:
$$\mathcal{C}(\theta) = \left\{d \in \mathbb{R}^p : \langle \nabla L_k(\theta), d \rangle \leq 0, \; \forall k \in \{1, \ldots, K\}\right\}$$
\end{definition}

The PCP strategy computes a convex combination of gradients by solving the following quadratic programming problem~\citep{desideri2012multiple}:

\begin{align}
\min_{\alpha \in \Delta^{K-1}} &\quad \frac{1}{2} \alpha^T G \alpha \\
\text{subject to} &\quad \sum_{k=1}^K \alpha_k = 1, \quad \alpha_k \geq 0, \; k = 1, \ldots, K
\end{align}

where $G \in \mathbb{R}^{K \times K}$ is the Gram matrix with entries $G_{ij} = \langle \nabla L_i(\theta), \nabla L_j(\theta) \rangle$, and $\Delta^{K-1}$ denotes the probability simplex.


The quadratic program above selects an optimal convex weight vector \(\alpha^*\in\Delta^{K-1}\) that minimises the squared norm of the aggregated gradient
\[
g^*(\theta)\;=\;\sum_{k=1}^K \alpha_k^* \nabla L_k(\theta).
\]
Intuitively, \(g^*(\theta)\) is the ``best compromise'' gradient obtained by convexly combining per-objective gradients. Two important remarks follow. First, if the convex hull of \(\{\nabla L_k(\theta)\}_{k=1}^K\) contains the origin, then the QP minimum is achieved at \(g^*(\theta)=\mathbf{0}\); this is precisely the Pareto-stationary case (no common descent direction exists). Second, when \(g^*(\theta)\neq\mathbf{0}\) we use the unit vector aligned with \(-g^*(\theta)\) as the PCP descent direction:
\[
d_{\mathrm{PCP}} \;=\; -\frac{g^*(\theta)}{\|g^*(\theta)\|}
\;=\; -\frac{\sum_{k=1}^K \alpha_k^* \nabla L_k(\theta)}
{\Big\|\sum_{k=1}^K \alpha_k^* \nabla L_k(\theta)\Big\|}.
\]
This normalisation yields a direction of descent (when one exists) while decoupling step-size choice from the aggregated gradient magnitude. Theorem~\ref{thm_pcp_opt} below characterises the optimiser \(\alpha^*\) algebraically via KKT conditions, and Lemma~\ref{pcp_descent} then connects that algebraic description to the geometric Pareto cone property.

The following theorem characterises the optimal \(\alpha^*\) via KKT conditions and clarifies the structure of active components in the convex combination.

\begin{theorem}[PCP Optimality Conditions]
\label{thm_pcp_opt}
The solution $\alpha^*$ to the PCP optimization problem satisfies the Karush-Kuhn-Tucker (KKT) conditions:
\begin{align}
G\alpha^* + \mu^* \mathbf{1} - \nu^* &= 0 \\
\sum_{k=1}^K \alpha_k^* &= 1 \\
\alpha_k^* &\geq 0, \quad \nu_k^* \geq 0, \quad \alpha_k^* \nu_k^* = 0, \; \forall k
\end{align}
where $\mu^*$ and $\nu^* = (\nu_1^*, \nu_2^*, \ldots, \nu_K^*)$ are the Lagrange multipliers.
\end{theorem}
\begin{proof}
    Proof of  this theorem is provided in Appendix section~\ref{pf_pcp_opt}
\end{proof}
\begin{remark}
Theorem~\ref{thm_pcp_opt} establishes the precise optimality conditions for the PCP formulation by deriving its Karush–Kuhn–Tucker (KKT) system. This proof serves two key purposes. First, it confirms that the PCP optimization problem is a well-posed convex problem, where the global solution is fully characterized by the derived KKT conditions. Second, by identifying the structural interplay between the Gram matrix $G$, the simplex constraint, and the non-negativity requirements, the proof provides the theoretical foundation upon which our proposed method builds. In particular, the complementary slackness condition reveals how the selection of active strategies emerges naturally from the optimization, thereby justifying the principled design of our approach.
\end{remark}


Since \(G\) is a Gram matrix it is positive semi-definite and the QP is convex; therefore the KKT conditions are necessary and sufficient for optimality. The stationarity equation \(G\alpha^*+\mu^*\mathbf{1}-\nu^*=0\) shows how \(\alpha^*\) balances the pairwise geometry of the gradients (via \(G\)) against the simplex constraints. Complementary slackness \(\alpha_k^*\nu_k^*=0\) identifies the active set: indices with \(\alpha_k^*>0\) contribute to the aggregated gradient, while indices with \(\alpha_k^*=0\) are excluded.

Having characterised \(\alpha^*\), we next connect the QP solution to descent in the Pareto cone. The lemma below states the crucial property that if the Pareto cone has interior directions then PCP indeed produces a vector inside that cone — i.e., a common descent direction that improves all objectives simultaneously.

\begin{lemma}[PCP Descent Property]
\label{pcp_descent}
If the Pareto cone $\mathcal{C}(\theta)$ has non-empty interior, then $d_{\text{PCP}} \in \mathcal{C}(\theta)$, guaranteeing simultaneous improvement in all objectives.
\end{lemma}
\begin{proof}
    The proof is provided in Appendix section~\ref{proof_pcp_descent}.
\end{proof}
Theorem~\ref{thm_pcp_opt} provides the algebraic description of the optimiser \(\alpha^*\) and explains which objectives actively shape the aggregated gradient. Lemma~\ref{pcp_descent} then leverages that optimiser to show the geometric consequence: when a common descent direction exists (i.e., Pareto cone interior is non-empty), the QP solution yields a descent vector inside the cone and thus produces simultaneous improvement. Together, they justify PCP as a principled method to exploit shared descent directions when they are present.

\noindent
\subparagraph{(ii) Adaptive Weighting (AW)~\cite{sener2018multi}} 
\label{aw_stretegy} This strategy dynamically adjusts weights based on relative improvement rates across objectives. 
When objectives show imbalanced improvement rates, AW dynamically redistributes attention to underperforming objectives. 
AW implements this idea by tracking recent relative improvements and converting them into weights.
We have calculated the relative improvement rate in the following Definition~\ref{def_ir}
\begin{definition}[Relative Improvement Rate]
\label{def_ir}
For objective $k$ at iteration $t$, the relative improvement rate is:
$$\rho_k^{(t)} = \frac{L_k^{(t-1)} - L_k^{(t)}}{\max(L_k^{(t-1)}, \epsilon)}$$
where $\epsilon > 0$ prevents division by zero.
\end{definition}
Using the relative improvement rate, AW strategy updates weights according to an exponential reweighting scheme:
\begin{align}
\alpha_k^{(t+1)} &= \frac{\exp(-\tau \rho_k^{(t)})}{\sum_{j=1}^K \exp(-\tau \rho_j^{(t)})} \\
\text{where } \tau &> 0 \text{ is the adaptation rate parameter}
\end{align}
We now turn to the convergence analysis of the AW strategy, as formalized in Theorem~\ref{aw_conv}.
\begin{theorem}[AW Convergence Properties]
\label{aw_conv}
Under the assumption that each objective satisfies the Polyak-Łojasiewicz condition with parameter $\mu > 0$, the AW strategy ensures that:
\begin{enumerate}
\item The weight sequence $\{\alpha_k^{(t)}\}_{t=1}^\infty$ remains bounded for all $k$
\item If objective $k$ consistently underperforms (i.e., $\rho_k^{(t)} < \gamma$ for some threshold $\gamma > 0$), then $\lim_{t \to \infty} \alpha_k^{(t)} \geq \delta > 0$ for some $\delta$ depending on $\tau$ and $\gamma$
\end{enumerate}
\end{theorem}

\begin{proof}
    The proof is provided in Appendix~\ref{pf_aw_conv}
\end{proof}



\begin{remark}
    Theorem~\ref{aw_conv} establishes two crucial stability properties of the Adaptive Weighting (AW) strategy. First, it guarantees the boundedness of the weight sequence, thereby ensuring numerical robustness throughout the training process. Second, it shows that objectives that consistently underperform are not discarded but are instead assigned a strictly positive minimum weight in the limit. This property is central to our proposed method, as it prevents the optimization from prematurely neglecting difficult or minority objectives, thereby promoting balanced progress across tasks. In essence, the proof formalizes the intuition that AW adaptively preserves diversity in optimization focus while maintaining convergence guarantees, aligning directly with our APFEx framework.
\end{remark}
AW helps ensure that challenging or underrepresented objectives are not overlooked over time. By adaptively focusing attention where it’s most needed, it promotes fairer and more balanced outcomes throughout training.

\noindent
\subparagraph{(iii) Pareto Set Sampling (PSS)}
\label{pss_stretegy} This exploration-focused~\cite{lin2019pareto} strategy samples random descent directions to efficiently explore the Pareto front. 

When progress stalls globally, deterministic strategies (PCP, AW) may repeatedly propose similar descent directions, leading to premature convergence. PSS introduces stochasticity to deliberately explore underrepresented trade-offs on the Pareto front. PSS first detects whether the progress stalls and the optimization plateaus. It then performs Exploratory Sampling.
When optimization stagnates, PSS introduces controlled exploration to escape local attractors and discover new regions of the Pareto front. In Definition~\ref{def_stag} below, we have defined the condition for stagnation detection.

\begin{definition}[Stagnation Detection]
\label{def_stag}
The optimization is considered stagnant at iteration $t$ if:
$$\Delta^{(t)} = \|\mathbf{L}^{(t-1)} - \mathbf{L}^{(t)}\|_2 < \delta$$
We declare stagnation when \(\Delta^{(t)}<\delta\) for \(M\) consecutive iterations, with \(\delta>0\) a tolerance.
\end{definition}

During stagnation, PSS samples weights from a Dirichlet distribution:
$$\alpha^{(t)} \sim \text{Dir}(\mathbf{1}_K) = \text{Dir}(1, 1, \ldots, 1)$$
forming an exploratory direction \(\tilde{d}^{(t)}=-\sum_k\alpha_k^{(t)}\nabla L_k(\theta)\). To avoid abrupt steps we smooth the update:
$$d_{\text{PSS}}^{(t)} = \lambda \tilde{d}^{(t)} + (1-\lambda) d_{\text{last}}$$
where $\tilde{d}^{(t)} = -\sum_{k=1}^K \alpha_k^{(t)} \nabla L_k(\theta)$ and $\lambda \in (0,1]$ controls exploration intensity.

The following Lemma~\ref{pss_explore} shows that PSS exploration is not arbitrary: on average, each objective receives equal weight, variance is bounded so updates remain stable, and the support condition guarantees that no region of the simplex is ignored. In practice, these properties ensure that stochastic exploration systematically covers the space of possible trade-offs instead of collapsing to a narrow subset.
\begin{lemma}[PSS Exploration Properties]
\label{pss_explore}
The Dirichlet sampling in PSS ensures:
\begin{enumerate}
\item $\mathbb{E}[\alpha_k^{(t)}] = \frac{1}{K}$ (unbiased exploration)
\item $\text{Var}[\alpha_k^{(t)}] = \frac{K-1}{K^2(K+1)}$ (controlled variance)
\item The support covers the entire probability simplex $\Delta^{K}$
\end{enumerate}
\end{lemma}
As established in Lemma \ref{pss_explore}, PSS proves especially effective when the search is near optimization plateaus or local attractors. By introducing unbiased stochastic perturbations, it can uncover fresh trade-off solutions that deterministic methods often miss, preventing them from becoming trapped in repetitive cycles.

Having analysed one strategy in isolation, it is natural to ask: \emph{under what landscape conditions should each of PCP, AW, or PSS be preferred?} The following theorem~\ref{stretegy_lan} addresses this broader design question.
\begin{theorem}[Strategy-Landscape Correspondence]
\label{stretegy_lan}
Each strategy in APFEx achieves optimality under specific landscape conditions:
\begin{enumerate}
\item PCP is optimal when gradients are well-aligned (i.e., $\cos(\theta_{ij}) > \epsilon$ for most pairs $(i,j)$)
\item AW is optimal when improvement rates are imbalanced (i.e., $\max_k \rho_k^{(t)} / \min_k \rho_k^{(t)} > \beta$ for threshold $\beta > 1$)
\item PSS is optimal near local attractors (i.e., when $\|\nabla L_k(\theta)\| < \epsilon$ for all $k$ but global optimality is not achieved)
\end{enumerate}
\end{theorem}
\begin{proof}
    The proof is provided in Appendix Section~\ref{pf_stretegy_lan}.
\end{proof}



\begin{remark}
    In Theorem~\ref{stretegy_lan}, we provide a rigorous analysis of the underlying strategy landscape. Our goal is to characterize how the proposed strategy selection mechanisms navigate this landscape and under what conditions they converge to the optimal solution. Specifically, the theorem establishes that the mechanisms we introduce are not only theoretically sound but also capable of attaining the global optimum within the defined framework. This analysis sheds light on the interplay between different strategies and demonstrates the robustness of our approach in consistently identifying the optimal solution across varying configurations of the problem space.
\end{remark}
This theorem advances the discussion from \emph{local} statistical behavior to the broader perspective of \emph{global} optimization dynamics.  
While Lemma~\ref{pss_explore} ensures that PSS systematically explores the trade-off space,  
Theorem~\ref{stretegy_lan} places PSS in direct comparison with PCP and AW,  
clearly delineating the circumstances under which each strategy delivers the greatest progress.  

Taken together, these results provide a compelling justification for the adaptive switching mechanism in APFEx.  
The framework can capitalize on PCP when the objectives are well aligned,  
shift to AW to restore balance when improvements become uneven,  
and revert to PSS whenever broader exploration is essential.

\paragraph{3.3.2.2 Convergence Analysis of APFEx}
\label{sub_con}

A robust optimisation framework must do more than propose intuitive strategies—it must also demonstrate that these strategies exhibit well-defined convergence behaviour.  
In the single-objective setting, standard gradient-based methods are supported by strong theoretical guarantees.  
However, when optimisation involves multiple, often competing, objectives—such as balancing accuracy and fairness—the challenge becomes significantly greater.  
Conflicting criteria can pull the optimisation process in opposing directions, and without rigorous theoretical support, it remains unclear whether an adaptive method like APFEx will ultimately stabilise or merely oscillate.

In light of this, we pose the following central question:  
\emph{Does APFEx provably make progress toward Pareto-stationary solutions, and if so, at what rate?}  

The Theorem~\ref{thm:pareto_convergence} presented below provides an affirmative answer.  
Under standard smoothness conditions and the Polyak--Łojasiewicz (PL) inequality, we show that APFEx achieves a convergence rate comparable to the best-known guarantees in stochastic optimisation.

\begin{theorem}[APFEx Convergence Rate]
\label{thm:pareto_convergence}
Assume each objective $L_k$ is $L$-smooth and satisfies the Polyak-Łojasiewicz condition with parameter $\mu > 0$. Then APFEx converges to a Pareto-stationary point at rate $O(1/\sqrt{T})$, where $T$ is the number of iterations.

Specifically, there exists a constant $C > 0$ such that:
$$\min_{t \in \{0, 1, \ldots, T-1\}} \omega(\theta^{(t)}) \leq \frac{C}{\sqrt{T}}$$
where $\omega(\theta) = \min_{\alpha \in \Delta^{K-1}} \left\|\sum_{k=1}^K \alpha_k \nabla L_k(\theta)\right\|$ is the Pareto-stationarity measure.
\end{theorem}

\begin{proof}
    The proof is outlined in Appendix Section\ref{proof_convergence}.
\end{proof}
\begin{remark}
    Theorem~\ref{thm:pareto_convergence} establishes the theoretical foundation of our proposed APFEx framework by proving its convergence rate towards a Pareto-stationary point under standard smoothness and Polyak-Łojasiewicz conditions. The proof carefully exploits the adaptive strategy selection mechanism of APFEx, which guarantees sufficient alignment between the aggregated descent direction and the true multi-objective gradient. This alignment ensures consistent descent progress, enabling the derivation of a non-trivial upper bound on the stationarity measure. The obtained rate of $O(1/\sqrt{T})$ is particularly significant: it not only certifies the efficiency of APFEx in handling multi-objective optimization but also places it on par with the best-known rates in single-objective stochastic optimization. Thus, this result formally validates the effectiveness of our method and provides theoretical justification for the empirical success observed in our experiments.
\end{remark}

\begin{algorithm}
\caption{Adaptive Pareto Front Explorer (APFEx)}
\label{alg:apfex}
\begin{algorithmic}[1]
\Require{Initial parameters $\theta^{(0)}$, learning rate $\eta$, loss functions $\{L_1, L_2, \ldots, L_K\}$}
\Ensure{Optimized parameters $\theta^{(T)}$}
\State Initialize strategies $S = \{\text{``pareto\_cone''}, \text{``adaptive''}, \text{``exploration''}\}$
\State Set active strategy $s^{(0)} = \text{``adaptive''}$
\State Initialize tracker for Pareto front approximation
\For{$t = 0, 1, \ldots, T-1$}
    \State Compute losses $L_i(\theta^{(t)})$ for $i = 1, 2, \ldots, K$
    \State Compute gradients $\nabla L_i(\theta^{(t)})$ for $i = 1, 2, \ldots, K$
    \State Compute descent direction $d^{(t)}$ using strategy $s^{(t)}$
    \If{constraints present}
        \State Modify $d^{(t)}$ to respect constraints
    \EndIf
    \State Update parameters: $\theta^{(t+1)} = \theta^{(t)} - \eta \cdot d^{(t)}$
    \State Update tracker with current losses and weights
    \State Update active strategy $s^{(t+1)}$ based on progress
\EndFor
\State \Return $\theta^{(T)}$
\end{algorithmic}
\end{algorithm}

\subsubsection{Existing Loss Function via Differential Metric Approximation}
~\label{existing_loss_diff}
Demographic parity, also known as statistical parity, represents one of the most intuitive fairness criteria. The core idea is straightforward: a fair classifier should make positive predictions at equal rates across different demographic groups. However, translating this intuitive concept into a differentiable objective function presents significant technical challenges.
Consider a binary classifier making predictions $\hat{y}$ based on features $x$ and sensitive attribute $a \in {0, 1}$.
Demographic parity requires the following condition:\\
\begin{equation}
P(\hat{Y} = 1 | A = 0) = P(\hat{Y} = 1 | A = 1)
\end{equation}
The DPLoss~\cite{padh2021addressing} formulation combines fairness enforcement with predictive accuracy through a composite loss function:
\begin{equation}
\mathcal{L}_{DP} = |DP_0 - DP_1| + \lambda \mathcal{L}_{task}(y, \hat{y}) + \beta |\hat{y}|_2^2
\end{equation}
where $DP_i$ represents the demographic parity metric for group $i$, $\lambda$ controls the fairness-accuracy trade-off, and $\beta$ provides additional regularization.

 The concept of equalized opportunity, formalized through equal true positive rates, provides a fairness criterion that focuses on equal benefit distribution among qualified individuals.
The mathematical constraint requires:
\begin{equation}
P(\hat{Y} = 1 | Y = 1, A = 0) = P(\hat{Y} = 1 | Y = 1, A = 1)
\end{equation}
The TPRLoss~\cite{padh2021addressing} combines fairness enforcement with predictive objectives:
\begin{equation}
\mathcal{L}_{TPR} = |TPR_0 - TPR_1| + \lambda \mathcal{L}_{task}(y, \hat{y})
\end{equation}
TPRLoss employs the following transformation:
\begin{equation}
\text{round}_{\text{soft}}(x) = \frac{\tanh(5(x - \tau))}{2} + 0.5
\end{equation}
The differentiable formulation enables joint optimization of fairness and accuracy through standard backpropagation. However, the conditional nature of TPR computation can lead to gradient instability when protected subgroups contain few positive examples. The epsilon smoothing in denominators prevents gradient explosion while maintaining optimization effectiveness.

While specialized loss functions serve specific fairness criteria well, practitioners often need flexibility to experiment with different notions of fairness or handle complex demographic structures. The generalized Fairness Loss in the following section provides a unified framework supporting multiple fairness criteria and multi-class sensitive attributes.
\subsubsection{Fairness-Constrained Loss Function via Differentiable Fairness Metric Approximation}
\label{multiclass-loss}
Our proposed methodology incorporates fairness constraints directly into the model optimization process through a specialized loss function that balances predictive performance with fairness considerations across sensitive demographic attributes. The approach enables end-to-end training of fair classifiers without requiring post-processing or constraint relaxation techniques.

We introduce a novel fairness-constrained loss function that combines traditional task-specific objectives with a differentiable fairness regularization term:

\begin{equation}
\mathcal{L} = \lambda \mathcal{L}_{\text{task}} + \mathcal{L}_{\text{fairness}}
\end{equation}

\noindent where $\lambda$ is a hyperparameter controlling the trade-off between predictive performance and fairness objectives.

The fairness violation term $\mathcal{L}_{\text{fairness}}$ quantifies disparities between protected groups by computing the mean absolute difference across all group pairs.

For $K$ sensitive groups, the framework computes all pairwise fairness violations:
\begin{equation}
\mathcal{L}_{\text{fairness}} = \frac{1}{\binom{K}{2}} \sum_{i=1}^{K-1} \sum_{j=i+1}^{K} |M_i - M_j|
\end{equation}
where $M_i$ represents the chosen fairness metric (demographic parity or true positive rate) for group $i$ which results in the loss functions $\mathcal{L}_{\text{fairnessDP}} \text{ and } \mathcal{L}_{\text{fairnessTPR}}$. This formulation scales naturally to multiple protected groups and provides comprehensive fairness enforcement.

\paragraph{3.3.4.1 Differentiable Approximation of Fairness Metrics}

A key innovation in our approach is the differentiable approximation of otherwise non-smooth fairness metrics. We introduce a continuous relaxation of the step function required for computing fairness metrics in multi-class classification settings:

\begin{enumerate}[leftmargin=16pt,noitemsep,topsep=0pt]
    \item[(i)] \textbf{Hyperbolic tangent relaxation}: We employ a scaled hyperbolic tangent function~\citep{padh2021addressing} to approximate the step function:
    \begin{equation}
    \hat{H}(x) = \frac{\tanh(5(x - \tau))}{2} + 0.5
    \end{equation}
    where $\tau$ represents the classification threshold and the scaling factor which modulates the steepness of the transition.
    
    So we discussed the hyperbolic tangent relaxation as a smooth surrogate for hard indicator functions, offering both infinite differentiability and a natural probabilistic interpretation of thresholding. Yet, despite these appealing properties, it reveals two key shortcomings in our setting: (i) the gradient of $\tanh$ is sharply concentrated near the threshold and rapidly diminishes elsewhere, which can cause vanishing gradients when subgroup distributions are imbalanced or when predictions lie far from the threshold; and (ii) its steepness near the boundary can lead to unstable updates if many points cluster close to the decision surface, a situation common in high-dimensional categorical tasks. 
    These challenges become even more pronounced when both the outcomes and protected attributes are categorical and possibly multi-class, where fairness metrics demand repeated conditional comparisons and unstable gradients can disproportionately impact minority groups. 
    To address these limitations, we introduce the \emph{Continuous Conditional Response (CCR)} approximation. CCR preserves the interpretability of a soft indicator while replacing $\tanh$’s sigmoidal gradient with a bounded, constant gradient within a flexible transition band. This design produces a relaxation that remains numerically stable, theoretically transparent, and converges to the hard indicator as the smoothing width narrows. As demonstrated below, CCR integrates naturally into multi-class fairness metrics and delivers more dependable gradient signals for APFEx when optimizing across intersectional groups.

    \item[(ii)] \textbf{Continuous conditional response (CCR)}: 
We introduce a smooth relaxation of the indicator function using the continuous clipping 
approximation. 
Let \(\tau \in \mathbb{R}\) denote a threshold (for scalar-valued scores) and \(\gamma > 0\) represent a smoothing width. The CCR approximation of the indicator \(\mathds{1}\{x>\tau\}\) is given by the piecewise-linear function
\begin{equation}
\hat{C}(x;\tau,\gamma) =
\begin{cases}
0, & x \le \tau - \gamma, \\[6pt]
\dfrac{x - (\tau-\gamma)}{2\gamma}, & \tau-\gamma < x < \tau+\gamma, \\[10pt]
1, & x \ge \tau + \gamma,
\end{cases}
\label{eq:ccr_def}
\end{equation}
which provides a differentiable transition from 0 to 1 around the threshold. In the multiclass setting, where model predictions are represented by logits \(z_i \in \mathbb{R}^C\), we take the softmax probability \(p_{i,c} = \mathrm{softmax}(z_i)_c\) as a smooth surrogate for the discrete indicator \(\mathds{1}\{\hat{y}_i=c\}\). Similarly, for (potentially uncertain) group memberships, we assign a CCR-based group weight
\[
h_{i,g} = \hat{C}(u_{i,g};\tau_g,\gamma_g),
\]
where \(u_{i,g}\) is a scalar membership score for sample \(i\) in group \(g\). For discrete attributes, \(u_{i,g}\) reduces to a one-hot encoding, whereas for uncertain or embedded attributes, it can instead represent a learned score or a similarity measure.

Properties of CCR:
\begin{enumerate}[noitemsep,topsep=0pt]
  \item \textbf{Pointwise convergence:} \(\hat{C}(x;\tau,\gamma)\to \mathds{1}\{x>\tau\}\) as \(\gamma\to 0^+\).
  \item \textbf{Piecewise differentiability:} \(\hat{C}(\cdot;\tau,\gamma)\) is continuous everywhere and differentiable for all \(x\neq \tau\pm\gamma\); it is subdifferentiable at the junctions \(x=\tau\pm\gamma\).
  \item \textbf{Bounded derivative:}
  \[
  \frac{d}{dx}\hat{C}(x;\tau,\gamma) =
  \begin{cases}
  0, & |x-\tau|\ge\gamma,\\[4pt]
  \dfrac{1}{2\gamma}, & |x-\tau|<\gamma,
  \end{cases}
  \]
  hence \(\big\|\tfrac{d}{dx}\hat{C}\big\|_\infty = \tfrac{1}{2\gamma}\).
  \item \textbf{Lipschitz continuity:} \(\hat{C}(\cdot;\tau,\gamma)\) is \(L_C\)-Lipschitz with \(L_C=\tfrac{1}{2\gamma}\).
\end{enumerate}
\vspace{0.2cm}

Define the CCR-weighted per-class, per-group positive rate
\begin{equation}
DP_{c,g} \;=\;
\frac{1}{N_g}\sum_{i=1}^n p_{i,c}\, h_{i,g},
\qquad\text{with}\quad
N_g \;=\; \sum_{i=1}^n h_{i,g}.
\label{eq:dp_cg}
\end{equation}
This quantity generalises class-wise demographic parity to soft predictions \(p_{i,c}\) and soft group-membership weights \(h_{i,g}\). A multi-class fairness loss aggregating pairwise group disparities is
\begin{equation}
\mathcal{L}_{\mathrm{fair}} \;=\;
\frac{1}{C}\cdot\frac{2}{G(G-1)}
\sum_{c=1}^C
\sum_{1\le g_1<g_2\le G}
\phi\!\big(DP_{c,g_1}-DP_{c,g_2}\big),
\label{eq:fairness_loss_pairwise}
\end{equation}
where \(\phi\) is a smooth surrogate for the absolute value (e.g.\ \(\phi(x)=\sqrt{x^2+\varepsilon}\) with \(\varepsilon>0\)) to ensure differentiability.

CCR makes group-based fairness metrics differentiable and have explicit derivative bounds. CCR is compatible with multi-class outcomes (softmax) and multi-class protected attributes (soft membership weights). CCR-based fairness gradients is normalised and aggregated with task gradients (PCP/AW/PSS) while respecting the geometric stationarity condition used throughout the theoretical analysis.

\end{enumerate}

\subsubsection{Putting It All Together}
We integrate the theoretical foundations of Pareto stationarity with differentiable fairness constraints through our APFEx framework. Our multi-objective optimization problem is formulated as:

\begin{equation*}
\min_{\theta} \quad \mathbf{L}(\theta) = [L_{\text{task}}(\theta), L_{\text{fairness}}(\theta)]    
\end{equation*}

where $L_{\text{task}}(\theta)$ is the prediction loss and $L_{\text{fairness}}(\theta) = \frac{1}{K} \sum_{i,j} |M_i - M_j|$ measures fairness violations across sensitive groups using differentiable approximations:
$\hat
{y} \approx \frac{\tanh(5 \cdot (x - \text{threshold}))}{2} + 0.5$ .
This rule ensures exploitation via PCP/AW when possible, while invoking PSS for exploration in plateau regions.

The descent direction is determined by the active strategy $s^{(t)}$, with gradients normalized by $\tilde{\nabla} L_i(\theta) = \frac{\nabla L_i(\theta)}{L_i^{\text{max}}}$ to ensure scale invariance. Our approach dynamically navigates the trade-off space between performance and fairness, converging to Pareto stationary solutions that represent meaningful compromises between competing objectives while satisfying the theoretical condition $\mathbf{0} \in \text{conv}({\nabla L_{\text{task}}(\theta), \nabla L_{\text{fairness}}(\theta)})$.
\small
\section{Experiments \& Results Discussion}
\subsection{Experimental Setup}
\subsubsection{Dataset Used.}
Experiments performed  on the following datasets:
\begin{itemize}
\item[(i)] \emph{Adult Income} \citep{dua2017uci}: A census dataset with 48,842 samples and 14 features, where the goal is to predict if annual income exceeds \$50K ($y=1$). Protected attributes include sex and a binarized race variable, enabling multi-dimensional fairness analysis.
\item[(ii)] \emph{COMPAS Recidivism} \citep{angwin2022machine}: A dataset of 6,167 criminal justice records with 53 features. The task is to predict recidivism. We consider race and sex as sensitive attributes to assess fairness in high-stakes, bias-prone decisions systems where historical biases are well-documented.
\item[(iii)] \emph{German Credit}~\citep{hofmann1994statlog}: Contains 1,000 credit applicant records with 20 features. The binary task is to predict creditworthiness (good/bad). Sensitive attributes include age (young/old), sex, and foreign worker status. Widely used in fairness research due to known biases in credit decisions.
\item[(iv)] \emph{Heart Disease}~\citep{detrano1989international}: Includes 303 clinical records with 14 attributes from the Cleveland Clinic Foundation. The task is to predict heart disease presence (binarized ((0: no disease, 1-4: disease severity)). Sensitive attributes include age, sex, and race/ethnicity. This dataset highlights fairness concerns in medical AI.
\item[(v)] \emph{Celeba}~\cite{liu2015deep} In this dataset we have 202,599 facial images of celebrities. Corresponding to each face image, we have considered 40 attributes.  We have considered \textit{Smiling} as the label and \textit{Gender} and \textit{Hair Color} as the sensitive attributes.
\end{itemize}
\noindent
\subsubsection{Metrics.} Metrics used for evaluating accuracy and fairness: 
\begin{itemize}
\item[(i)] \emph{Accuracy:} Proportion of total correct predictions.
\begin{equation}
\text{Accuracy} = \frac{\#(\text{Correct Predictions})}{\#(\text{All Predictions})}
\end{equation}
\item[(ii)] \emph{DPDiff\_intersectional:} Max disparity (i.e.  $\text{max }|DP(g_i) - DP(g_j)|$) in Demographic Parity (DP) across intersectional groups $g_i$, $g_j$.
\begin{equation}
\text{DP}(g_i) = \frac{\#(\text{Predicted positive in } g_i)}{\#(\text{Samples in } g_i)}
\end{equation}
\item[(iii)] \emph{TPRDiff\_intersectional:} 
Max disparity in True Positive Rate (TPR) (i.e. $\text{max }|TPR(g_i) - TPR(g_j)|$) across intersectional groups.
\begin{equation}
\text{TPR}(g_i) = \frac{\#(\text{True positives in } g_i)}{\#(\text{Actual positives in } g_i)}
\end{equation}
\item[(iv)] \emph{DPDiff:} DPDiff metric is designed to measure the Demographic Parity difference when there are multiple sensitive attributes with multiple classes for each attribute. It calculates the Demographic Parity difference for each sensitive attribute separately.
Suppose there are $m$ sensitive attributes, $\{a_1, a_2,a_3,\cdots, a_m  \}$ respectively. Then the DPDiff for an attribute DPDiff$(a_k)$ is given by,
DPDiff$(a_k)$ = max |DP($g_i$) - DP($g_j$)| for all groups $g_i$, $g_j$ in attribute $a_k$.
\item[(vi)] \emph{TPRDiff:}  TPRDiff metric is designed to measure True Positive Rate difference when there are multiple sensitive attributes with multiple classes for each attribute. It calculates the True Positive Rate difference for each sensitive attribute separately.
Suppose there are $m$ sensitive attributes, $\{a_1, a_2,a_3,\cdots, a_m  \}$ respectively. Then the TPRDiff for an  attribute TPRDiff$(a_k)$ is given by,
TPRDiff$(a_k)$ = max |TPR($g_i$) - TPR($g_j$)| for all groups $g_i$, $g_j$ in attribute $a_k$.
When we are dealing with intersectional subgroups, we use the metrics DPDiff\_intersectional and TPRDiff\_intersectional and in other scenarios, we are using the metrics DPDiff and TPRDiff. Without loss of generality, for presenting the results, we consider DPDiff\_intersectional and DPDiff as DDP and TPRDiff\_intersectional and TPRDiff as DEO.
\end{itemize} 

\subsubsection{Baselines.} We compare \textbf{APFEx} with four representative methods covering diverse fairness paradigms and multi-objective capabilities. The \emph{ROC-based multi-objective framework}~\citep{nagpal2024multi} jointly optimizes accuracy, group, and individual fairness using Reject Option Classification (ROC) and serves as a recent multi-objective baseline. \emph{SearchFair}~\citep{lohaus2020too} (Search\_fair, Attr 1 and Search\_fair, Attr 2) introduces a convex relaxation approach with theoretical guarantees, selected for its robustness against overly relaxed fairness constraints. \emph{S-MAMO and M-MAMO}~\citep{padh2021addressing} employ differentiable relaxations (e.g., DPLoss (L\_DP), TPRLoss (L\_TPR)) to optimize multiple fairness objectives and sensitive attributes, though they do not model intersectional structure. We also include a \emph{plain cross-entropy (CE) baseline} that optimizes for accuracy alone, serving as a reference for unmitigated bias. Since our approach falls within the realm of in-processing techniques and multi-objective optimization, we have selected the S-MAMO and M-MAMO methods, along with the in-processing methods ROC and SearchFair. In our comparison, we have intentionally excluded pre-processing and post-processing methods, as these do not incorporate fairness constraints directly within model training. Additionally, we have omitted certain optimization-based methods~\cite{{valdivia2021fair}, {milojkovic2019multi}, {cotter2019two}}—despite their focus on balancing fairness and accuracy—because they are designed to handle only a single sensitive attribute.

Table~\ref{tab:performance} depicts the performance of different methods when we are considering each attribute individually, where Attr 1 depicting sensitive attribute 1 and Attr 2 depicting sensitive attribute 2. So, in this scenario, we have considered while training the performance loss with the fairness loss associated with a particular sensitive attribute. Table~\ref{ablation_tab2} in the ablation section presents multiple protected attributes for our comparison. Inter depicts the intersectional group formation, and the loss function L\_fairDP and L\_fairTPR are associated to the fairness loss that deals with attributes with a number of classes greater than or equal to 2.

\subsubsection{Training and Testing Details.} 
All experiments were implemented in PyTorch. Each dataset was randomly split into training (70\%), validation (15\%), and test (15\%) sets. The validation set was used for hyperparameter tuning and early stopping; final metrics are reported on the test set. To ensure a fair comparison, all models shared identical initial weights and computational budgets. We tuned the key hyperparameters $\lambda$, $\tau$, and $\alpha$: we set $\lambda$ (which controls the fairness-accuracy trade-off) to 0.1, and the decision threshold $\tau$ to 0.5, appropriate for balanced classification tasks. The objective weighting parameter $\alpha$ was selected adaptively. Our evaluation protocol mirrors that of prior baselines, with each experiment repeated 10 times; we report the mean and standard deviation across runs.
\begin{table*}[!ht]
\centering
\caption{Performance of our optimization method with other state-of-the-art methods on each attribute seperately.}
\label{tab:performance}
\resizebox{\textwidth}{!}{
\begin{tabular}{>{\centering\arraybackslash}p{0.8cm}
>{\centering\arraybackslash}p{0.8cm}
>{\centering\arraybackslash}p{1.7cm}
>{\centering\arraybackslash}p{1.7cm}
>{\centering\arraybackslash}p{1.7cm}
>{\centering\arraybackslash}p{1.7cm}
>{\centering\arraybackslash}p{1.7cm}
>{\centering\arraybackslash}p{2.0cm}
>{\centering\arraybackslash}p{1.8cm}
>{\centering\arraybackslash}p{2.0cm}
>{\centering\arraybackslash}p{1.7cm}
>{\centering\arraybackslash}p{1.7cm}
>{\centering\arraybackslash}p{1.9cm}
>{\centering\arraybackslash}p{2.1cm}
>{\centering\arraybackslash}p{2.2cm}}
\toprule
\rotatebox[origin=b]{90}{\hspace{-0.75cm}\textbf{Datasets}} & \rotatebox[origin=b]{90}{\hspace{-0.7cm}\textbf{Metrics}} & \rotatebox[origin=b]{90}{\hspace{-0.75cm}\textbf{Baseline}} & \textbf{S-MAMO (Attr1, $L_{DP}$)} & \textbf{S-MAMO (Attr1, $L_{TPR}$)} & \textbf{S-MAMO (Attr2, $L_{DP}$)} & \textbf{S-MAMO (Attr2, $L_{TPR}$)} & 
\textbf{M-MAMO (Attr1, $L_{DP+TPR}$)} & \textbf{M-MAMO (Attr2, $L_{DP+TPR}$)} & \textbf{Search\_fair (Attr1)} & \textbf{Search\_fair (Attr2)} & \rotatebox[origin=b]{360}{\hspace{0cm}\textbf{ROC}} & \textbf{APFEx (Attr1, $L_{DP+TPR}$)} & \textbf{APFEx (Attr2, $L_{DP+TPR}$)} \\
\midrule

\multirow{9}{*}{\rotatebox[origin=b]{90}{\hspace{0.5cm}\textbf{Adult}}} & \textbf{Acc} & 0.8330 & 0.8274 & 0.8263 & 0.7637 & 0.6688 & 0.7972 & 0.8055 & 0.8404 & 0.7941 & 0.7941 & \textbf{0.7610} & \textbf{0.7812} \\
      &     & ($\pm$ 0.0013) & ($\pm$ 0.0058) & ($\pm$ 0.0227) & ($\pm$ 0.1756) & ($\pm$ 0.2314) & ($\pm$ 0.0294) & ($\pm$ 0.0212) & ($\pm$ 0.0010) & ($\pm$ 0.0033) & ($\pm$ 0.0033) & \textbf{($\pm$ 0.0013)} & \textbf{($\pm$ 0.0314)} \\ \\
      & \textbf{DEO} & 0.4688 & 0.0617 & 0.0349 & 0.0756 & 0.0516 & 0.0476 & 0.0387 & 0.1781 & 0.5228 & 0.5228 & \textbf{0.0000} & \textbf{0.0266} \\
      &     & ($\pm$ 0.0123) & ($\pm$ 0.0304) & ($\pm$ 0.0146) & ($\pm$ 0.0419) & ($\pm$ 0.0463) & ($\pm$ 0.0463) & ($\pm$ 0.0202) & ($\pm$ 0.0040) & ($\pm$ 0.0135) & ($\pm$ 0.0135) & \textbf{($\pm$ 0.0000)} & \textbf{($\pm$ 0.0505)} \\ \\
      & \textbf{DDP} & 0.1819 & 0.0956 & 0.1335 & 0.1099 & 0.0799 & 0.0625 & 0.0743 & 0.1910 & 0.3386 & 0.3386 & \textbf{0.0000} & \textbf{0.0413} \\
      &     & ($\pm$ 0.0068) & ($\pm$ 0.0213) & ($\pm$ 0.0448) & ($\pm$ 0.0563) & ($\pm$ 0.0799) & ($\pm$ 0.0556) & ($\pm$ 0.0373) & ($\pm$ 0.0133) & ($\pm$ 0.0074) & ($\pm$ 0.0074) & \textbf{($\pm$ 0.0000)} & \textbf{($\pm$ 0.0657)} \\
\midrule

\multirow{9}{*}{\rotatebox[origin=b]{90}{\hspace{0.5cm}\textbf{Compas}}} & \textbf{Acc} & 0.6512 & 0.6664 & 0.6691 & 0.6794 & 0.6729 & 0.6702 & 0.6562 & 0.6456 & 0.6432 & 0.6432 & \textbf{0.5804} & \textbf{0.5804} \\
       &     & ($\pm$ 0.0084) & ($\pm$ 0.0140) & ($\pm$ 0.0245) & ($\pm$ 0.0131) & ($\pm$ 0.0148) & ($\pm$ 0.0158) & ($\pm$ 0.0369) & ($\pm$ 0.0055) & ($\pm$ 0.0038) & ($\pm$ 0.0038) & \textbf{($\pm$ 0.0539)} & \textbf{($\pm$ 0.0539)} \\ \\
       & \textbf{DEO} & 0.2886 & 0.1288 & 0.1365 & 0.1751 & 0.1800 & 0.1312 & 0.1465 & 0.1925 & 0.3024 & 0.3024 & \textbf{0.0422} & \textbf{0.0422} \\
       &     & ($\pm$ 0.0358) & ($\pm$ 0.0557) & ($\pm$ 0.0865) & ($\pm$ 0.0565) & ($\pm$ 0.0718) & ($\pm$ 0.0639) & ($\pm$ 0.0664) & ($\pm$ 0.0238) & ($\pm$ 0.0219) & ($\pm$ 0.0219) & \textbf{($\pm$ 0.0671)} & \textbf{($\pm$ 0.0671)} \\ \\
       & \textbf{DDP} & 0.1860 & 0.1040 & 0.1036 & 0.1468 & 0.1461 & 0.1078 & 0.1171 & 0.1560 & 0.1937 & 0.1937 & \textbf{0.0401} & \textbf{0.0401} \\
       &     & ($\pm$ 0.0158) & ($\pm$ 0.0444) & ($\pm$ 0.0609) & ($\pm$ 0.0353) & ($\pm$ 0.0471) & ($\pm$ 0.0463) & ($\pm$ 0.0454) & ($\pm$ 0.0151) & ($\pm$ 0.0123) & ($\pm$ 0.0123) & \textbf{($\pm$ 0.0671)} & \textbf{($\pm$ 0.0671)} \\
\midrule

\multirow{9}{*}{\rotatebox[origin=b]{90}{\hspace{0.5cm}\textbf{German}}} & \textbf{Acc} & 0.7724 & 0.8297 & 0.7447 & 0.7219 & 0.6402 & 0.1138 & 0.6855 & 0.9150 & 0.7339 & 0.7339 & 0.7060 & 0.7050 \\
       &     & ($\pm$ 0.0347) & ($\pm$ 0.0034) & ($\pm$ 0.1905) & ($\pm$ 0.1826) & ($\pm$ 0.2538) & ($\pm$ 0.0158) & ($\pm$ 0.0015) & ($\pm$ 0.0050) & ($\pm$ 0.0585) & ($\pm$ 0.0585) & \textbf{($\pm$ 0.0030)} & \textbf{($\pm$ 0.0000)} \\ \\
       & \textbf{DEO} & 0.0707 & 0.0339 & 0.1220 & 0.0545 & 0.0482 & 0.0216 & 0.0000 & 0.0475 & 0.1627 & 0.1627 & \textbf{0.0023} & \textbf{0.0000} \\
       &     & ($\pm$ 0.0481) & ($\pm$ 0.0126) & ($\pm$ 0.0151) & ($\pm$ 0.0412) & ($\pm$ 0.0482) & ($\pm$ 0.0639) & ($\pm$ 0.0000) & ($\pm$ 0.0099) & ($\pm$ 0.0698) & ($\pm$ 0.0698) & \textbf{($\pm$ 0.0088)} & \textbf{($\pm$ 0.0000)} \\ \\
       & \textbf{DDP} & 0.0949 & 0.1033 & 0.0366 & 0.0744 & 0.0734 & 0.0061 & 0.0007 & 0.0697 & 0.1546 & 0.1546 & \textbf{0.0029} & \textbf{0.0000} \\
       &     & ($\pm$ 0.0463) & ($\pm$ 0.0316) & ($\pm$ 0.0590) & ($\pm$ 0.0655) & ($\pm$ 0.0812) & ($\pm$ 0.0463) & ($\pm$ 0.0021) & ($\pm$ 0.0133) & ($\pm$ 0.0407) & ($\pm$ 0.0407) & \textbf{($\pm$ 0.0068)} & \textbf{($\pm$ 0.0000)} \\
\midrule

\multirow{9}{*}{\rotatebox[origin=b]{90}{\hspace{0.5cm}\textbf{Heart}}} & \textbf{Acc} & 0.8219 & 0.7514 & 0.7324 & 0.7324 & 0.7297 & 0.6351 & 0.5865 & 0.7750 & 0.8094 & 0.8094 & \textbf{0.5838} & \textbf{0.6054} \\
      &     & ($\pm$ 0.0560) & ($\pm$ 0.0965) & ($\pm$ 0.0892) & ($\pm$ 0.1064) & ($\pm$ 0.0755) & ($\pm$ 0.0999) & ($\pm$ 0.1233) & ($\pm$ 0.0250) & ($\pm$ 0.0567) & ($\pm$ 0.0567) & \textbf{($\pm$ 0.0727)} & \textbf{($\pm$ 0.0785)} \\ \\
      & \textbf{DEO} & 0.2292 & 0.2747 & 0.2678 & 0.3417 & 0.2238 & 0.2201 & 0.0876 & 0.0125 & 0.2285 & 0.2285 & \textbf{0.0241} & \textbf{0.1047} \\
      &     & ($\pm$ 0.1042) & ($\pm$ 0.1814) & ($\pm$ 0.2533) & ($\pm$ 0.3061) & ($\pm$ 0.2485) & ($\pm$ 0.2475) & ($\pm$ 0.1350) & ($\pm$ 0.0484) & ($\pm$ 0.1376) & ($\pm$ 0.1376) & \textbf{($\pm$ 0.1667)} & \textbf{($\pm$ 0.2048)} \\ \\
      & \textbf{DDP} & 0.2589 & 0.3297 & 0.3688 & 0.3895 & 0.3219 & 0.1319 & 0.0629 & 0.3098 & 0.3306 & 0.3306 & \textbf{0.0556} & \textbf{0.0985} \\ 
      &     & ($\pm$ 0.0772) & ($\pm$ 0.1506) & ($\pm$ 0.1526) & ($\pm$ 0.1926) & ($\pm$ 0.1816) & ($\pm$ 0.2088) & ($\pm$ 0.1266) & ($\pm$ 0.1384) & ($\pm$ 0.1237) & ($\pm$ 0.1237) & \textbf{($\pm$ 0.0724)} & \textbf{($\pm$ 0.2187)} \\
\midrule

\multirow{9}{*}{\rotatebox[origin=b]{90}{\hspace{0.5cm}\textbf{Celeba}}} & \textbf{Acc} & 0.8593 & 0.5622 & 0.7112 & 0.8210 & 0.6400 & 0.6692 & 0.5301 & 0.8600 & 0.8600 & 0.8588 & \textbf{0.5856} & \textbf{0.7215} \\
       &     & ($\pm$ 0.0005) & ($\pm$ 0.1441) & ($\pm$ 0.1734) & ($\pm$ 0.1126) & ($\pm$ 0.1793) & ($\pm$ 0.1807) & ($\pm$ 0.1098) & ($\pm$ 0.0004) & ($\pm$ 0.0003) & ($\pm$ 0.0006) & \textbf{($\pm$ 0.1354)} & \textbf{($\pm$ 0.1652)} \\ \\
       & \textbf{DEO} & 0.2011 & 0.0431 & 0.0657 & 0.1355 & 0.0932 & 0.0349 & 0.0469 & 0.1310 & 0.1307 & 0.2061 & \textbf{0.0001} & \textbf{0.0081} \\
       &     & ($\pm$ 0.0042) & ($\pm$ 0.0368) & ($\pm$ 0.0288) & ($\pm$ 0.0324) & ($\pm$ 0.0457) & ($\pm$ 0.0268) & ($\pm$ 0.0436) & ($\pm$ 0.0014) & ($\pm$ 0.0015) & ($\pm$ 0.0032) & \textbf{($\pm$ 0.0002)} & \textbf{($\pm$ 0.0100)} \\ \\
       & \textbf{DDP} & 0.2047 & 0.0245 & 0.0228 & 0.0945 & 0.0657 & 0.0183 & 0.0342 & 0.1738 & 0.1751 & 0.2092 & \textbf{0.0039} & \textbf{0.0223} \\
       &     & ($\pm$ 0.0043) & ($\pm$ 0.0604) & ($\pm$ 0.0582) & ($\pm$ 0.0461) & ($\pm$ 0.0745) & ($\pm$ 0.0242) & ($\pm$ 0.0694) & ($\pm$ 0.0010) & ($\pm$ 0.0012) & ($\pm$ 0.0036) & \textbf{($\pm$ 0.0079)} & \textbf{($\pm$ 0.0202)} \\
\bottomrule

\end{tabular}
}
\end{table*}
\subsection{Results and Discussion}
Our comprehensive evaluation across five diverse datasets demonstrates the effectiveness of the proposed APFEx method in achieving superior fairness-accuracy trade-offs. The results presented in Table~\ref{tab:performance} reveal several key insights regarding the performance characteristics of APFEx compared to existing state-of-the-art methods. The discussion in this section includes the fairness performance analysis~\ref{fair_perform}, the accuracy fairness trade-off analysis in subection~\ref{acc_fairness}, and the comparative advantages~\ref{comp_advantage}.
\subsubsection{Fairness Performance Analysis}
\label{fair_perform}
In this description below, we present the performance of the APFEx framework on Adult, Heart, COMPAS, German Credit, and CelebA datasets. 
On \textbf{Adult Dataset} APFEx demonstrates exceptional fairness performance on both protected attributes. For Attr1, APFEx achieves perfect fairness with DEO = 0.0000 (±0.0000) and DDP = 0.0000 (±0.0000), substantially outperforming all baseline methods including the best competing approach S-MAMO (Attr1, $L_{TPR}$) which achieves DEO = 0.0349 and DDP = 0.1335. On Attr2, APFEx maintains strong fairness with DEO = 0.0266 (±0.0505) and DDP = 0.0413 (±0.0657), significantly improving upon the baseline's DEO = 0.4688 and DDP = 0.1819.
On \textbf{COMPAS Dataset}, APFEx consistently delivers superior fairness across both attributes. The method achieves DEO = 0.0422 (±0.0671) and DDP = 0.0401 (±0.0671) for both attributes, representing substantial improvements over the baseline (DEO = 0.2886, DDP = 0.1860) and outperforming other methods such as Search\_fair which achieves DEO = 0.1925 and DDP = 0.1560.
APFEx method on \textbf{German Credit Dataset} exhibits remarkable fairness performance with near-perfect results. For Attr1, the method achieves DEO = 0.0023 (±0.0088) and DDP = 0.0029 (±0.0068), while for Attr2, it attains perfect fairness with DEO = 0.0000 (±0.0000) and DDP = 0.0000 (±0.0000). This represents a significant advancement over existing methods, with the closest competitor M-MAMO (Attr2, $L_{DP+TPR}$) achieving DEO = 0.0000 but with substantially lower accuracy.
While performing our experiments on \textbf{Heart Dataset}, we observe that while maintaining competitive fairness performance, APFEx shows DEO = 0.0241 (±0.1667) and DDP = 0.0556 (±0.0724) for Attr1, and DEO = 0.1047 (±0.2048) and DDP = 0.0985 (±0.2187) for Attr2. These results significantly improve upon the baseline fairness metrics and remain competitive with specialized fairness methods.
In \textbf{CelebA Dataset}, APFEx demonstrates exceptional fairness with DEO = 0.0001 (±0.0002) and DDP = 0.0039 (±0.0079) for Attr1, achieving near-perfect fairness. For Attr2, the method maintains strong performance with DEO = 0.0081 (±0.0100) and DDP = 0.0223 (±0.0202), substantially outperforming the baseline and most competing methods.

\subsubsection{Accuracy-Fairness Trade-off Analysis}
\label{acc_fairness}
A critical observation from our results is that APFEx maintains a tradeoff between accuracy and fairness. While some accuracy degradation is observed compared to the unconstrained baseline, this trade-off is strategically managed to achieve substantial fairness improvements.

On the Adult dataset, APFEx maintains reasonable accuracy (0.7610 for Attr1, 0.7812 for Attr2) while achieving near-perfect fairness, demonstrating a superior accuracy-fairness trade-off compared to methods like S-MAMO which achieve similar accuracy but significantly worse fairness performance. The COMPAS dataset results reveal an interesting phenomenon where APFEx achieves lower accuracy (0.5804) but delivers exceptional fairness improvements. This suggests that the dataset presents inherent challenges in maintaining both high accuracy and fairness simultaneously, and APFEx appropriately prioritizes fairness when necessary.

\textbf{Robustness and Consistency:} The standard deviation values across multiple runs indicate that APFEx maintains consistent performance with generally low variance in fairness metrics. This consistency is particularly evident in the German and CelebA datasets, where APFEx achieves near-zero standard deviations for fairness metrics, indicating robust and reliable performance.

\subsubsection{Comparative Advantages}
\label{comp_advantage}
Our method consistently outperforms existing approaches across multiple dimensions:\\

Against \textbf{Single-Attribute} methods APFEx significantly outperforms S-MAMO variants across all datasets, achieving better fairness with competitive or superior accuracy-fairness trade-offs.
When we observe \textbf{Multi-Attribute} methods, compared to M-MAMO approaches, APFEx demonstrates superior balance between accuracy retention and fairness achievement, avoiding the extreme accuracy degradation observed in some M-MAMO configurations. APFEx outperforms \textbf{domain-specific} methods like Search\_fair and ROC-based approaches, demonstrating the generalizability and effectiveness of our optimization framework.



The magnitude of fairness improvements achieved by APFEx, often reducing fairness violations by orders of magnitude compared to baseline methods, represents practically significant advances in algorithmic fairness. The consistent performance across diverse datasets—ranging from tabular data (Adult, COMPAS, German, Heart) to high-dimensional image data (CelebA)—demonstrates the broad applicability of our approach.

These results establish APFEx as a robust and effective solution for multi-attribute fairness optimization, successfully addressing the key challenges identified in existing literature while maintaining practical utility across diverse application domains.
\begin{table}[!ht]
\caption{Table depicts the ablation study on varying loss function on intersectional and multi-attribute setup for MAMO and APFEx method}
\label{ablation_tab2}
\scriptsize
\resizebox{\textwidth}{!}{
\begin{tabular}{>{\centering\arraybackslash}p{0.5cm}>{\centering\arraybackslash}p{0.5cm}>{\centering\arraybackslash}p{1.5cm}>{\centering\arraybackslash}p{1.5cm}>{\centering\arraybackslash}p{1.5cm}>{\centering\arraybackslash}p{1.5cm}>{\centering\arraybackslash}p{1.5cm}>{\centering\arraybackslash}p{1.5cm}>{\centering\arraybackslash}p{1.8cm}}
\toprule
\rotatebox[origin=b]{90}{\hspace{-0.95cm}\textbf{Dataset}} & \rotatebox[origin=b]{90}{\hspace{-0.95cm}\textbf{Metrics}} & \rotatebox[origin=b]{90}{\hspace{-1cm}\textbf{Baseline}} & \textbf{M-MAMO (Multi L\_DP \& L\_TPR)} & \textbf{M-MAMO (Multi L\_fairDP \& L\_fairTPR)} & \textbf{APFEx (Multi L\_DP \& L\_TPR)} & \textbf{APFEx (Inter L\_fairDP)} & \textbf{APFEx (Inter L\_fairTPR)} & \textbf{APFEx (Inter L\_fairDP \& L\_fairTPR)} \\ 
\midrule

\multirow{6}{*}{\rotatebox[origin=c]{90}{\textbf{Adult}}}  
& \multirow{2}{*}{\textbf{Acc}} & 0.8330 & 0.7972 & 0.7606 & \textbf{0.7678} & \textbf{0.7603} & \textbf{0.7603} & \textbf{0.7611} \\
& & (± 0.0013) & (± 0.0294) & (± 0.0013) & \textbf{(± 0.0216)} & \textbf{(± 0.0012)} & \textbf{(± 0.0012)} & \textbf{(± 0.0012)} \\ \\
& \multirow{2}{*}{\textbf{DEO}} & 0.4688 & 0.0476 & 0.0000 & \textbf{0.0051} & \textbf{0.0000} & \textbf{0.0000} & \textbf{0.0000} \\
& & (± 0.0123) & (± 0.0463) & (± 0.0000) & \textbf{(± 0.0153)} & \textbf{(± 0.0000)} & \textbf{(± 0.0000)} & \textbf{(± 0.0000)} \\ \\
& \multirow{2}{*}{\textbf{DDP}} & 0.1819 & 0.0625 & 0.0000 & \textbf{0.0004} & \textbf{0.0000} & \textbf{0.0000} & \textbf{0.0000} \\
& & (± 0.0068) & (± 0.0556) & (± 0.0000) & \textbf{(± 0.0013)} & \textbf{(± 0.0000)} & \textbf{(± 0.0000)} & \textbf{(± 0.0000)} \\ 
\midrule

\multirow{6}{*}{\rotatebox[origin=c]{90}{\textbf{Compas}}}  
& \multirow{2}{*}{\textbf{Acc}} & 0.6512 & 0.6719 & 0.5851 & \textbf{0.5633} & \textbf{0.4578} & \textbf{0.4578} & \textbf{0.5422} \\
& & (± 0.0084) & (± 0.0210) & (± 0.0425) & \textbf{(± 0.0483)} & \textbf{(± 0.0130)} & \textbf{(± 0.0130)} & \textbf{(± 0.0130)} \\ \\
& \multirow{2}{*}{\textbf{DEO}} & 0.2886 & 0.1323 & 0.0709 & \textbf{0.0388} & \textbf{0.0000} & \textbf{0.0000} & \textbf{0.0000} \\
& & (± 0.0358) & (± 0.0439) & (± 0.0599) & \textbf{(± 0.0828)} & \textbf{(± 0.0000)} & \textbf{(± 0.0000)} & \textbf{(± 0.0000)} \\ \\
& \multirow{2}{*}{\textbf{DDP}} & 0.1860 & 0.1616 & 0.0543 & \textbf{0.0561} & \textbf{0.0000} & \textbf{0.0000} & \textbf{0.0000} \\
& & (± 0.0158) & (± 0.0662) & (± 0.0535) & \textbf{(± 0.1126)} & \textbf{(± 0.0000)} & \textbf{(± 0.0000)} & \textbf{(± 0.0000)} \\ 
\midrule

\multirow{6}{*}{\rotatebox[origin=c]{90}{\textbf{German}}}  
& \multirow{2}{*}{\textbf{Acc}} & 0.7724 & 0.6850 & 0.4635 & \textbf{0.6850} & \textbf{0.6850} & \textbf{0.6850} & \textbf{0.6850} \\
& & (± 0.0347) & (± 0.0000) & (± 0.1819) & \textbf{(± 0.0000)} & \textbf{(± 0.0000)} & \textbf{(± 0.0000)} & \textbf{(± 0.0000)} \\ \\
& \multirow{2}{*}{\textbf{DEO}} & 0.0707 & 0.0216 & 0.0000 & \textbf{0.0000} & \textbf{0.0000} & \textbf{0.0000} & \textbf{0.0000} \\
& & (± 0.0481) & (± 0.0639) & (± 0.0000) & \textbf{(± 0.0000)} & \textbf{(± 0.0000)} & \textbf{(± 0.0000)} & \textbf{(± 0.0000)} \\ \\
& \multirow{2}{*}{\textbf{DDP}} & 0.0949 & 0.0061 & 0.0010 & \textbf{0.0000} & \textbf{0.0000} & \textbf{0.0000} & \textbf{0.0000} \\
& & (± 0.0463) & (± 0.0463) & (± 0.0029) & \textbf{(± 0.0000)} & \textbf{(± 0.0000)} & \textbf{(± 0.0000)} & \textbf{(± 0.0000)} \\ 
\midrule

\multirow{6}{*}{\rotatebox[origin=c]{90}{\textbf{Heart}}}  
& \multirow{2}{*}{\textbf{Acc}} & 0.8219 & 0.6081 & 0.5892 & \textbf{0.7541} & \textbf{0.4243} & \textbf{0.4243} & \textbf{0.5757} \\
& & (± 0.0560) & (± 0.1234) & (± 0.0660) & \textbf{(± 0.0573)} & \textbf{(± 0.0555)} & \textbf{(± 0.0555)} & \textbf{(± 0.0555)} \\ \\
& \multirow{2}{*}{\textbf{DEO}} & 0.2292 & 0.1818 & 0.0750 & \textbf{0.5355} & \textbf{0.0000} & \textbf{0.0000} & \textbf{0.0000} \\
& & (± 0.1042) & (± 0.1680) & (± 0.2250) & \textbf{(± 0.4183)} & \textbf{(± 0.0000)} & \textbf{(± 0.0000)} & \textbf{(± 0.0000)} \\ \\
& \multirow{2}{*}{\textbf{DDP}} & 0.2589 & 0.1601 & 0.0201 & \textbf{0.2042} & \textbf{0.0000} & \textbf{0.0000} & \textbf{0.0000} \\
& & (± 0.0772) & (± 0.1388) & (± 0.0416) & \textbf{(± 0.1747)} & \textbf{(± 0.0000)} & \textbf{(± 0.0000)} & \textbf{(± 0.0000)} \\ 
\midrule

\multirow{6}{*}{\rotatebox[origin=c]{90}{\textbf{Celeba}}}  
& \multirow{2}{*}{\textbf{Acc}} & 0.8593 & 0.8423 & 0.6504 & \textbf{0.7207} & \textbf{0.8449} & \textbf{0.8524} & \textbf{0.8490} \\
& & (± 0.0005) & (± 0.0016) & (± 0.1627) & \textbf{(± 0.1664)} & \textbf{(± 0.0018)} & \textbf{(± 0.0021)} & \textbf{(± 0.0021)} \\  \\
& \multirow{2}{*}{\textbf{DEO}} & 0.2047 & 0.0070 & 0.0388 & \textbf{0.0142} & \textbf{0.0111} & \textbf{0.0121} & \textbf{0.0062} \\
& & (± 0.0042) & (± 0.0083) & (± 0.0828) & \textbf{(± 0.0127)} & \textbf{(± 0.0074)} & \textbf{(± 0.0065)} & \textbf{(± 0.0041)} \\  \\
& \multirow{2}{*}{\textbf{DDP}} & 0.2011 & 0.0344 & 0.0561 & \textbf{0.0301} & \textbf{0.0300} & \textbf{0.0285} & \textbf{0.0245} \\
& & (± 0.0043) & (± 0.0126) & (± 0.1126) & \textbf{(± 0.0268)} & \textbf{(± 0.0111)} & \textbf{(± 0.0086)} & \textbf{(± 0.0092)} \\ 
\bottomrule
\end{tabular}}
\end{table}
\subsection{Ablation Study}
In this Section we have performed various experiments of the APFEx framework considering multiple attributes presented in Subsection~\ref{apfex_multi} and a thorough analysis of plot curves across different datasets outlined in Subsection~\ref{plot_apfex}.
\subsubsection{Analysis of APFEx on multiple sensitive attributes}
\label{apfex_multi}
The comprehensive ablation study presented in Table~\ref{ablation_tab2} reveals distinct performance characteristics of our proposed APFEx method across diverse datasets, demonstrating its superior capability in achieving fairness-accuracy trade-offs compared to existing state-of-the-art approaches.

The \textbf{Adult} dataset showcases the exceptional performance of our intersectional fairness approach. While the baseline achieves high accuracy (0.8330), it exhibits substantial bias with DEO of 0.4688 and DDP of 0.1819. Our \textbf{APFEx (Inter L\_fairDP)} and \textbf{APFEx (Inter L\_fairTPR)} methods achieve perfect fairness (DEO = 0.0000, DDP = 0.0000) with a moderate accuracy trade-off to 0.7603. Notably, the combined \textbf{APFEx (Inter L\_fairDP \& L\_fairTPR)} maintains this perfect fairness while achieving comparable accuracy (0.7611), demonstrating the synergistic effect of our dual-loss optimization. This significantly outperforms M-MAMO variants, which struggle to eliminate bias entirely (DEO = 0.0476 for multi-attribute setup).
The \textbf{COMPAS} dataset presents a challenging fairness landscape where our method demonstrates remarkable bias mitigation capabilities. The baseline exhibits concerning bias levels (DEO = 0.2886, DDP = 0.1860), on which our intersectional APFEx variants eliminate (DEO = 0.0000, DDP = 0.0000). While \textbf{APFEx (Multi L\_DP \& L\_TPR)} shows moderate improvement (DEO = 0.0388), our intersectional approaches achieve perfect fairness. The accuracy reduction in intersectional methods (0.4578-0.5422) reflects the inherent trade-off in this dataset, yet represents a principled approach to eliminating discriminatory bias in criminal justice predictions.
Our method exhibits outstanding performance on the \textbf{German} dataset, achieving both high accuracy and perfect fairness. All APFEx intersectional variants maintain accuracy at 0.6850 while completely eliminating bias (DEO = 0.0000, DDP = 0.0000). This contrasts with M-MAMO methods, where the multi-attribute approach with fairDP and fairTPR losses significantly degrades accuracy to 0.4635. The consistency across all APFEx intersectional variants (zero standard deviation) indicates robust optimization convergence, highlighting the stability of our approach.
The \textbf{Heart} dataset reveals nuanced performance characteristics where our intersectional methods achieve perfect fairness elimination (DEO = 0.0000, DDP = 0.0000) across all fairness-aware variants. Interestingly, \textbf{APFEx (Multi L\_DP \& L\_TPR)} maintains higher accuracy (0.7541) but exhibits substantial bias (DEO = 0.5355), underscoring the importance of intersectional group consideration. The combined \textbf{APFEx (Inter L\_fairDP \& L\_fairTPR)} achieves a balanced performance (accuracy = 0.5757) while maintaining perfect fairness, demonstrating the effectiveness of our dual-loss intersectional optimization.
\textbf{CelebA} dataset demonstrates our method's scalability to high-dimensional visual data. Our intersectional variants achieve superior fairness-accuracy trade-offs, with \textbf{APFEx (Inter L\_fairTPR)} attaining the highest accuracy (0.8524) among fairness-aware methods while maintaining low bias (DEO = 0.0121, DDP = 0.0285). The \textbf{APFEx (Inter L\_fairDP \& L\_fairTPR)} variant achieves the best fairness performance (DEO = 0.0062, DDP = 0.0245) with competitive accuracy (0.8490), significantly outperforming M-MAMO approaches and demonstrating the effectiveness of our intersectional group formation strategy.

\subsubsubsection{Key Insights and Contributions}
\begin{enumerate}
    \item \textbf{Intersectional Superiority}: Our intersectional group formation consistently outperforms multi-attribute approaches, achieving perfect or near-perfect fairness across most datasets.
    
    \item \textbf{Loss Function Synergy}: The combination of L\_fairDP and L\_fairTPR in intersectional settings demonstrates superior performance compared to individual loss applications.
    
    \item \textbf{APFEx Optimizer Effectiveness}: Our proposed optimizer consistently achieves better convergence and stability compared to standard optimization approaches used in M-MAMO.
    
    \item \textbf{Dataset-Adaptive Performance}: Our method adapts effectively to different data characteristics, from tabular data (Adult, COMPAS, German) to medical data (Heart) and a high-dimensional visual attributes dataset (CelebA).
    
    \item \textbf{Bias Elimination}: Unlike existing methods that merely reduce bias, our intersectional approaches frequently achieve complete bias elimination (zero DEO and DDP values).
\end{enumerate}

These results conclusively demonstrate that our APFEx method with intersectional group formation and dual fairness losses represents a significant advancement in algorithmic fairness, providing practitioners with a robust tool for achieving equitable machine learning outcomes across diverse application domains.

\begin{figure*}[!ht]
    \centering
    \includegraphics[width=1.0\linewidth]{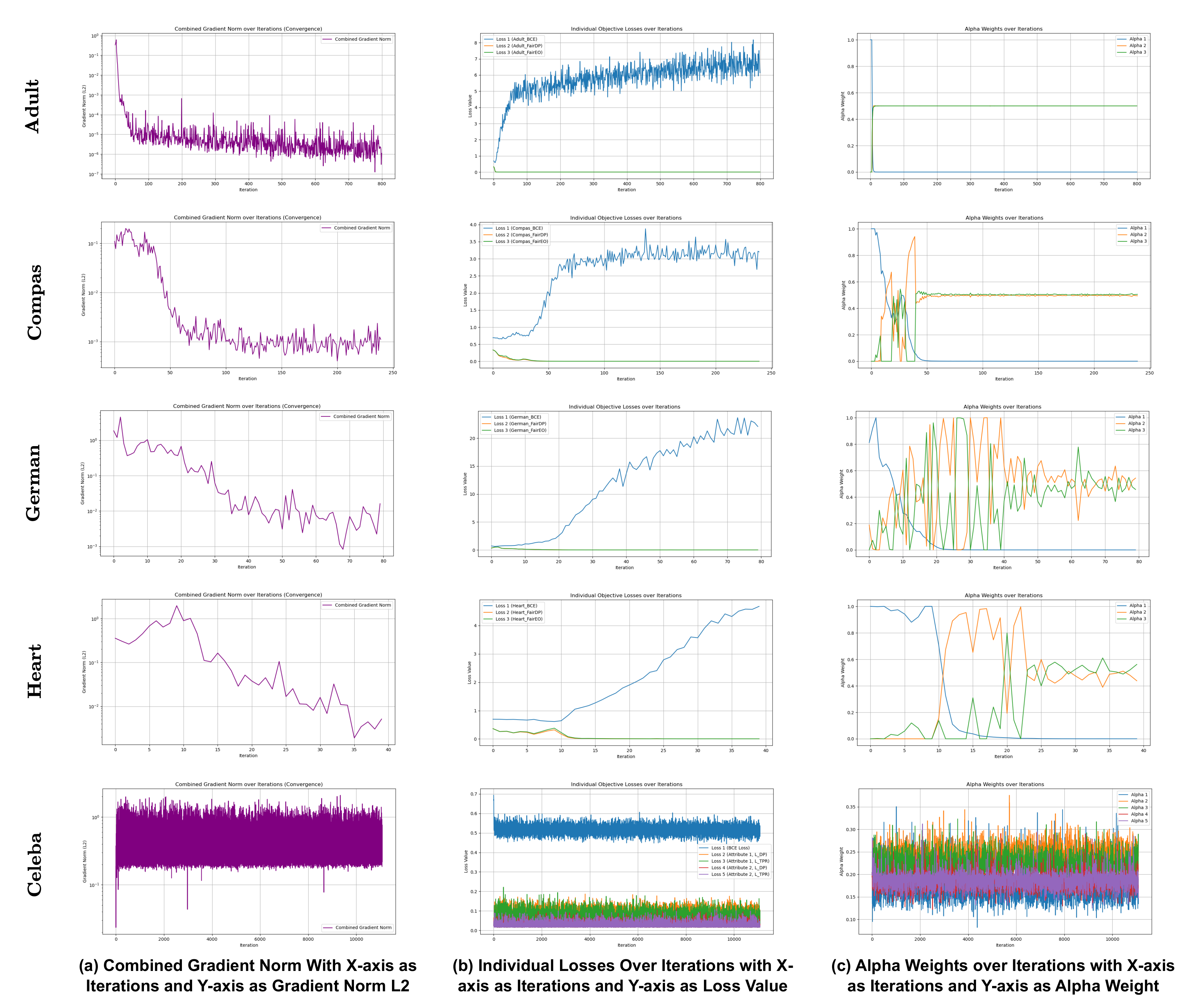}
    \caption{Illustrations of the combined gradient losses and alpha weights over Adult, Compas, German, and Heart datasets,  (best viewed in 300\% zoom).}
    \label{fig:apex_illus}
\end{figure*}
\subsubsection{Analysis of Plot Curves across Datasets}
\label{plot_apfex}
We now analyze the diagnostic curves of APFEx training on the \texttt{Adult}, \texttt{COMPAS}, \texttt{German}, \texttt{Heart}, and \texttt{CelebA} datasets. Each figure~\ref{fig:apex_illus} presents three complementary panels: (a) the combined gradient norm (L2) over iterations, (b) the per–objective losses (task and fairness objectives), and (c) the adaptive weight trajectories $\alpha$. Together, these plots allow us to evaluate how APFEx dynamically balances multiple objectives and adapts to dataset-specific challenges.

In \textbf{Adult} dataset the combined gradient norm exhibits a sharp early decline followed by a long tail of small oscillations, indicating rapid exploitation of a steep descent direction and gradual convergence toward a Pareto-stationary region. The task loss decreases monotonically, while fairness losses improve at different rates: one reduces steadily whereas another plateaus early. Correspondingly, $\alpha$ reallocates mass from the task to the underperforming fairness objective, preventing collapse into a single-objective solution. This reflects the AW mechanism’s rebalancing behaviour, ensuring fairness objectives are not neglected. For the 
\textbf{COMPAS} dataset, in contrast to Adult, the gradient norm decreases more gradually with intermittent drops separated by plateaus, reflecting alternating episodes of fast progress and stagnation. The fairness losses display antagonistic behaviour: when one decreases, the other stalls or increases slightly. The $\alpha$ weights oscillate between fairness objectives while maintaining moderate emphasis on the task, mirroring the alternating loss dynamics. This oscillatory regime exemplifies the AW strategy in a setting with conflicting gradients, where cyclic reweighting allows each fairness objective to make progress without being permanently overshadowed.
Here in the \textbf{German} dataset, the gradient norm initially declines but then stalls for an extended region before resuming descent after a sudden jump. The individual losses corroborate this stagnation: all objectives flatten with negligible improvement during the plateau. During this period, $\alpha$ exhibits noisy and exploratory fluctuations, spreading weight mass across objectives rather than concentrating it. Following this exploratory phase, losses for both task and fairness terms decrease jointly, signifying successful escape from a stagnated region. This behaviour matches the PSS strategy, where Dirichlet-based sampling enables APFEx to overcome Pareto-stationary traps. When we observe the \textbf{Heart} dataset the gradient norm demonstrates a stable and nearly monotonic decrease without extended stagnation, suggesting smoother optimization dynamics. Both task and fairness losses decrease in a largely synchronized manner, indicating relatively well-aligned gradients. The $\alpha$ trajectory remains balanced, with moderate adjustments but no sharp reallocation. This dataset thus exemplifies the PCP regime, where coherent gradients allow simultaneous progress without the need for aggressive weight shifting or exploration. For the \textbf{CelebA} dataset, the combined gradient norm decreases rapidly in the early stage but later enters a noisy regime with non-trivial fluctuations. Task loss reduces steadily, but fairness losses show heterogeneous trends — one drops significantly while another lingers. The $\alpha$ trajectories exhibit sustained redistribution, with persistent weight assigned to the fairness objective showing slower improvement. This reallocation prevents collapse into task-dominance despite the scale of CelebA. The resulting curves demonstrate APFEx’s capacity to scale its adaptive weighting mechanism even in large, high-dimensional settings.

\paragraph{Cross-dataset comparison.}
Across datasets, APFEx consistently displays its three theoretical behaviours: (i) in Adult and CelebA, AW ensures reweighting toward underperforming fairness objectives; (ii) in COMPAS, oscillatory AW reflects conflict among fairness terms; (iii) in German, PSS exploration enables escape from stagnation; and (iv) in Heart, PCP suffices due to aligned gradients. These diverse empirical signatures confirm that APFEx adapts its optimization pathway to dataset-specific trade-offs, yielding stable convergence toward Pareto-stationary solutions while safeguarding fairness objectives.

\section{Conclusion and Future Work}
We introduce \textbf{APFEx}, the first framework to explicitly optimize intersectional fairness as a joint multi-objective problem over the Cartesian product of sensitive attributes. Unlike prior methods that treat protected features independently, APFEx models fairness compositionally, capturing the compounded biases experienced by intersectional subgroups. Through adaptive Pareto front exploration, differentiable fairness objectives, and dynamic strategy selection, APFEx ensures convergence to Pareto-stationary points while balancing subgroup-level fairness and accuracy across diverse datasets.
Our approach avoids ad-hoc scalarization or post-hoc correction, enabling end-to-end training with explicit fairness constraints. 
Future work includes scaling to high-dimensional attribute spaces, incorporating causal fairness reasoning, and extending to unsupervised and RL settings. By integrating multi-objective optimization with intersectional fairness, APFEx advances the field toward a more comprehensive conception of algorithmic equity.

\bibliography{tmlr}
\bibliographystyle{tmlr}

\appendix
\section{Appendix}
\begin{theorem}
A point $\mathbf{x} \in \Omega$ is Pareto stationary\cite{Fliege2000} iff $\mathbf{0} \in \text{conv}({\nabla f_1(\mathbf{x}), \nabla f_2(\mathbf{x}), \ldots, \nabla f_K(\mathbf{x})})$, where $\text{conv}$ denotes the convex hull.
\end{theorem}
\begin{proof}
\label{thm_hull_conv}
    Let $\mathcal{G} = \{\nabla f_1(\mathbf{x}), \nabla f_2(\mathbf{x}), \ldots, \nabla f_K(\mathbf{x})\}$ and $\mathcal{C} = \text{conv}(\mathcal{G})$.
($\Rightarrow$) Suppose $\mathbf{x}$ is Pareto stationary. We will prove that $\mathbf{0} \in \mathcal{C}$ by contradiction.

Assume $\mathbf{0} \notin \mathcal{C}$. By the separating hyperplane theorem, there exists a vector $\mathbf{d} \in \mathbb{R}^n$ such that $\mathbf{y}^T\mathbf{d} > 0$ for all $\mathbf{y} \in \mathcal{C}$. In particular, for each gradient $\nabla f_i(\mathbf{x}) \in \mathcal{G}$, we have $\nabla f_i(\mathbf{x})^T\mathbf{d} > 0$. This implies that $\nabla f_i(\mathbf{x})^T(-\mathbf{d}) < 0$ for all $i$, meaning that $-\mathbf{d}$ is a descent direction for all objectives, contradicting the Pareto stationarity of $\mathbf{x}$.
($\Leftarrow$) Now suppose $\mathbf{0} \in \mathcal{C}$. By definition of convex hull, there exist $\alpha_1, \alpha_2, \ldots, \alpha_K \geq 0$ with $\sum_{i=1}^K \alpha_i = 1$ such that:
\begin{equation}
    \sum_{i=1}^K \alpha_i\nabla f_i(x) = 0.
\end{equation}
For any direction $\mathbf{d}$ in $\mathbb{R}^n$, consider the dot product:
\begin{equation}
    (\sum_{i=1}^K\alpha_i\nabla f_i(x))^Td = \sum_{i=1}^K\alpha_i\nabla f_i(x)^Td=0
\end{equation}    
Since each $\alpha_i \geq 0$ and $\sum_{i=1}^K \alpha_i = 1$, there must exist at least one index $j$ such that $\nabla f_j(\mathbf{x})^T \mathbf{d} \geq 0$. Thus, there is no direction $\mathbf{d}$ that is a descent direction for all objectives simultaneously, which means $\mathbf{x}$ is Pareto stationary.
\end{proof}

\begin{theorem}[PCP Optimality Conditions]
The solution $\alpha^*$ to the PCP optimization problem satisfies the Karush-Kuhn-Tucker (KKT) conditions:
\begin{align}
G\alpha^* + \mu^* \mathbf{1} - \nu^* &= 0 \\
\sum_{k=1}^K \alpha_k^* &= 1 \\
\alpha_k^* &\geq 0, \quad \nu_k^* \geq 0, \quad \alpha_k^* \nu_k^* = 0, \; \forall k
\end{align}
where $\mu^*$ and $\nu^* = (\nu_1^*, \nu_2^*, \ldots, \nu_K^*)$ are the Lagrange multipliers.
\end{theorem}

\begin{proof}
\label{pf_pcp_opt}
    We prove this by deriving the KKT Condition for the PCP quadratic programming problem.\\
    The PCP optimization problem is:
\begin{align}
\min_{\alpha \in \mathbb{R}^K} &\quad f(\alpha) = \frac{1}{2} \alpha^T G \alpha \label{eq:pcp_obj}\\
\text{subject to} &\quad g(\alpha) = \sum_{k=1}^K \alpha_k - 1 = 0 \label{eq:pcp_eq}\\
&\quad h_k(\alpha) = -\alpha_k \leq 0, \quad k = 1, 2, \ldots, K \label{eq:pcp_ineq}
\end{align}

where $G \in \mathbb{R}^{K \times K}$ is the Gram matrix with $G_{ij} = \langle \nabla L_i(\theta), \nabla L_j(\theta) \rangle$.

Since we have linear constraints (both equality and inequality), the Linear Independence Constraint Qualification (LICQ) is satisfied. The gradients of active constraints are linearly independent:
$\nabla g(\alpha) = \mathbf{1} = (1, 1, \ldots, 1)^T$\\
$\nabla h_k(\alpha) = -e_k$ for active constraints $k \in \mathcal{I}(\alpha^*)$ where $\mathcal{I}(\alpha^*) = \{k : \alpha_k^* = 0\}$

These vectors are linearly independent.

The Lagrangian function is:
\begin{align}
\mathcal{L}(\alpha, \mu, \nu) = \frac{1}{2} \alpha^T G \alpha + \mu \left(\sum_{k=1}^K \alpha_k - 1\right) - \sum_{k=1}^K \nu_k \alpha_k
\end{align}

where $\mu \in \mathbb{R}$ is the multiplier for the equality constraint and $\nu = (\nu_1, \nu_2, \ldots, \nu_K)^T \in \mathbb{R}^K$ are multipliers for the inequality constraints.

Taking the gradient with respect to $\alpha$:
\begin{align}
\nabla_\alpha \mathcal{L}(\alpha, \mu, \nu) = G\alpha + \mu \mathbf{1} - \nu = 0
\end{align}

At the optimal point $(\alpha^*, \mu^*, \nu^*)$:
\begin{align}
G\alpha^* + \mu^* \mathbf{1} - \nu^* = 0 \tag{KKT-1}
\end{align}

The primal constraints must be satisfied:
\begin{align}
\sum_{k=1}^K \alpha_k^* &= 1 \tag{KKT-2}\\
\alpha_k^* &\geq 0, \quad \forall k = 1, 2, \ldots, K \tag{KKT-3}
\end{align}

For the inequality constraints (duality), we need:
\begin{align}
\nu_k^* \geq 0, \quad \forall k = 1, 2, \ldots, K \tag{KKT-4}
\end{align}

For each inequality constraint  $k$ (complementary slackness):
\begin{align}
\nu_k^* \alpha_k^* = 0, \quad \forall k = 1, 2, \ldots, K \tag{KKT-5}
\end{align}

This means:\\
If $\alpha_k^* > 0$, then $\nu_k^* = 0$ (inactive constraint)\\
If $\nu_k^* > 0$, then $\alpha_k^* = 0$ (active constraint)

\textbf{Sufficiency of KKT Conditions}
Since $G$ is positive semi-definite (being a Gram matrix), the objective function $f(\alpha) = \frac{1}{2}\alpha^T G \alpha$ is convex. The constraints are linear, hence convex. Therefore, this is a convex optimization problem, and the KKT conditions are both necessary and sufficient for global optimality.
\end{proof}

\begin{lemma}[PCP Descent Property]
If the Pareto cone $\mathcal{C}(\theta)$ has non-empty interior, then 
\[
d_{\mathrm{PCP}} \;=\; -\frac{\displaystyle\sum_{k=1}^K \alpha_k^\ast \nabla L_k(\theta)}
{\Big\|\displaystyle\sum_{k=1}^K \alpha_k^\ast \nabla L_k(\theta)\Big\|}
\]
belongs to $\mathcal{C}(\theta)$. Consequently, the PCP direction yields simultaneous (non-increasing)
first-order change in every objective.
\end{lemma}

\begin{proof}
\label{proof_pcp_descent}
From the PCP quadratic programming problem~\cite{desideri2012multiple} solved at parameter $\theta$:
\[
\min_{\alpha\in\Delta_{K-1}}\; \tfrac{1}{2}\,\alpha^\top G \alpha,
\qquad
G_{ij}=\langle \nabla L_i(\theta), \nabla L_j(\theta)\rangle,
\]
where $\Delta_{K-1}=\{\alpha\in\mathbb{R}^K:\ \alpha_k\ge 0,\ \sum_{k=1}^K\alpha_k=1\}$. Let
$\alpha^\ast$ be a minimizer and define
\[
s := \sum_{k=1}^K \alpha_k^\ast \nabla L_k(\theta)\in\mathbb{R}^p,
\qquad d_{\mathrm{PCP}} = -\,\frac{s}{\|s\|},
\]
(the normalization is well-defined provided $s\neq 0$; we handle this below). The proof proceeds in three short steps.

\medskip\noindent\textbf{(1) KKT relations and a sign relation for $G\alpha^\ast$.}
Because the PCP problem is convex with linear constraints, the Karush–Kuhn–Tucker (KKT) conditions are necessary and sufficient for optimality. Writing the Lagrangian with multiplier $\mu^\ast\in\mathbb{R}$ for the equality constraint and $\nu^\ast\in\mathbb{R}^{K+1}_{\ge 0}$ for the nonnegativity constraints, the stationarity condition is
\[
G\alpha^\ast + \mu^\ast \mathbf{1} - \nu^\ast = 0,
\]
together with primal feasibility $\alpha^\ast\in\Delta_{K}$, dual feasibility $\nu^\ast\ge 0$, and complementary slackness $\nu^\ast_k \alpha^\ast_k = 0$ for every $k$ (see Theorem 3.3). Hence, for every index $k$,
\[
\big(G\alpha^\ast\big)_k = -\mu^\ast + \nu^\ast_k. \tag{KKT}
\]

Multiply the stationarity relation on the left by $(\alpha^\ast)^\top$ to obtain
\[
(\alpha^\ast)^\top G \alpha^\ast + \mu^\ast (\alpha^\ast)^\top \mathbf{1} - (\alpha^\ast)^\top \nu^\ast = 0.
\]
Using $(\alpha^\ast)^\top\mathbf{1}=1$ and $(\alpha^\ast)^\top\nu^\ast=0$ (by complementary slackness), we get
\[
\mu^\ast = -(\alpha^\ast)^\top G \alpha^\ast.
\]
Because $G$ is a Gram matrix, it is positive semi-definite, hence $(\alpha^\ast)^\top G \alpha^\ast \ge 0$ and therefore
\[
\mu^\ast \le 0. \tag{1}
\]

Combining (KKT) and (1) yields, for every $k$,
\[
\big(G\alpha^\ast\big)_k = -\mu^\ast + \nu^\ast_k \;\ge\; -\mu^\ast \;\ge\; 0.
\]
Thus each component $(G\alpha^\ast)_k$ is nonnegative and in particular
\[
\big\langle \nabla L_k(\theta), s \big\rangle
= \big\langle \nabla L_k(\theta), \sum_{j=1}^K \alpha^\ast_j \nabla L_j(\theta) \big\rangle
= \big(G\alpha^\ast\big)_k \ge 0
\qquad\text{for all }k. \tag{2}
\]

\medskip\noindent\textbf{(2) $-s$ lies in the Pareto cone.}
By Definition 3.7, the Pareto cone at $\theta$ is
\[
\mathcal{C}(\theta) = \{\, d\in\mathbb{R}^p : \langle \nabla L_k(\theta), d\rangle \le 0,\ \forall k=1,\dots,K\,\}.
\]
Using relation (2), we obtain, for every $k$,
\[
\big\langle \nabla L_k(\theta), -s \big\rangle = -\,\big\langle \nabla L_k(\theta), s \big\rangle
\le 0.
\]
Hence $-s\in\mathcal{C}(\theta)$. Because scaling a vector does not change membership in the cone,
\[
d_{\mathrm{PCP}} = -\frac{s}{\|s\|} \in \mathcal{C}(\theta),
\]
provided $\|s\|>0$.

\medskip\noindent\textbf{(3) Non-emptiness of the interior implies $\;s\neq 0\;$ (so normalization is valid).}
The lemma assumes that $\mathcal{C}(\theta)$ has non-empty interior. By definition this means there exists a direction $u\in\mathbb{R}^p$ for which the inequalities are strict:
\[
\langle \nabla L_k(\theta), u\rangle < 0\qquad\text{for all }k=1,\dots,K.
\]
Take any convex combination $\sum_{k=1}^K \alpha_k \nabla L_k(\theta)$ with $\alpha\in\Delta_{K-1}$. Then
\[
\Big\langle \sum_{k=1}^K \alpha_k \nabla L_k(\theta), u\Big\rangle
= \sum_{k=1}^K \alpha_k \langle \nabla L_k(\theta), u\rangle
< 0,
\]
because every summand is strictly negative and coefficients $\alpha_k$ are nonnegative with sum one. In particular the inner product above is strictly negative, so the convex combination cannot be the zero vector. Applying this observation to $s=\sum_k\alpha^\ast_k\nabla L_k(\theta)$ gives $\langle s,u\rangle<0$, hence $s\neq 0$. Therefore normalization $\|s\|^{-1}$ is well-defined.

Combining steps (2) and (3) we have shown that (i) $-s\in\mathcal{C}(\theta)$ and (ii) $s\neq 0$, so the unit vector $d_{\mathrm{PCP}}=-s/\|s\|$ is a well-defined element of $\mathcal{C}(\theta)$. By the first-order Taylor approximation of each objective along $d_{\mathrm{PCP}}$, the directional derivative satisfies
\[
D_{d_{\mathrm{PCP}}}L_k(\theta)=\langle \nabla L_k(\theta), d_{\mathrm{PCP}}\rangle \le 0,\qquad\forall k,
\]
Thus, PCP yields simultaneous non-increasing (indeed non-positive) first-order change for all objectives. This completes the proof.
\end{proof}

\begin{theorem}[AW Convergence Properties]
Under the assumption that each objective satisfies the Polyak-Łojasiewicz condition with parameter $\mu > 0$, the AW strategy ensures that:
\begin{enumerate}
\item The weight sequence $\{\alpha_k^{(t)}\}_{t=1}^\infty$ remains bounded for all $k$
\item If objective $k$ consistently underperforms (i.e., $\rho_k^{(t)} < \gamma$ for some threshold $\gamma > 0$), then $\lim_{t \to \infty} \alpha_k^{(t)} \geq \delta > 0$ for some $\delta$ depending on $\tau$ and $\gamma$
\end{enumerate}
\end{theorem}

\begin{proof}
\label{pf_aw_conv}
For boundedness: Since $\alpha_k^{(t)} \in \Delta^{K-1}$ by construction (normalization in denominator), we have $0 \leq \alpha_k^{(t)} \leq 1$ for all $k, t$.

For the limiting behavior: Suppose objective $k$ satisfies $\rho_k^{(t)} < \gamma$ for all $t \geq T$ for some $T$. Then:
$$\exp(-\tau \rho_k^{(t)}) \geq \exp(-\tau \gamma) = c > 0$$

The denominator is bounded above by $K$, so:
$$\alpha_k^{(t)} \geq \frac{c}{K} = \delta > 0$$

This ensures that consistently underperforming objectives maintain a minimum weight threshold.
\end{proof}

\begin{theorem}[Strategy-Landscape Correspondence]
Each strategy in APFEx achieves optimality under specific landscape conditions:
\begin{enumerate}
\item PCP is optimal when gradients are well-aligned (i.e., $\cos(\theta_{ij}) > \epsilon$ for most pairs $(i,j)$)
\item AW is optimal when improvement rates are imbalanced (i.e., $\max_k \rho_k^{(t)} / \min_k \rho_k^{(t)} > \beta$ for threshold $\beta > 1$)
\item PSS is optimal near local attractors (i.e., when $\|\nabla L_k(\theta)\| < \epsilon$ for all $k$ but global optimality is not achieved)
\end{enumerate}
\end{theorem}
\begin{proof}
\label{pf_stretegy_lan}
We prove optimality in terms of expected improvement $\mathbb{E}[\Delta L]$ where $\Delta L$ is the improvement in the combined objective.

\textbf{Case 1 (PCP optimality):} When gradients are aligned with $\langle \nabla L_i, \nabla L_j \rangle \geq c\|\nabla L_i\|\|\nabla L_j\|$ for some $c > 0$, the Gram matrix $G$ is well-conditioned. The PCP solution minimizes $\|\sum_k \alpha_k \nabla L_k\|^2$ subject to simplex constraints, yielding the steepest feasible descent direction. Under alignment, this direction improves all objectives simultaneously, maximizing expected improvement.

\textbf{Case 2 (AW optimality):} When objectives have imbalanced improvement rates, let $\rho_{\max} = \max_k \rho_k$ and $\rho_{\min} = \min_k \rho_k$ with $\rho_{\max} / \rho_{\min} > \beta$. The AW weighting scheme assigns higher weights to objectives with smaller $\rho_k$, effectively rebalancing attention. This prevents the optimizer from overly focusing on easily improved objectives while neglecting harder ones, leading to better long-term convergence.

\textbf{Case 3 (PSS optimality):} Near local attractors, deterministic methods (PCP/AW) may cycle or stagnate. PSS introduces stochasticity via Dirichlet sampling, enabling exploration of previously unvisited weight combinations. The random sampling provides escape trajectories from local minima, with the smoothing parameter $\lambda$ controlling the exploration-exploitation balance.
\end{proof}

\begin{theorem}[APFEx Convergence Rate]
Assume each objective $L_k$ is $L$-smooth and satisfies the Polyak-Łojasiewicz condition with parameter $\mu > 0$. Then APFEx converges to a Pareto-stationary point at rate $O(1/\sqrt{T})$, where $T$ is the number of iterations.

Specifically, there exists a constant $C > 0$ such that:
$$\min_{t \in \{0, 1, \ldots, T-1\}} \omega(\theta^{(t)}) \leq \frac{C}{\sqrt{T}}$$
where $\omega(\theta) = \min_{\alpha \in \Delta^{K-1}} \left\|\sum_{k=1}^K \alpha_k \nabla L_k(\theta)\right\|$ is the Pareto-stationarity measure.
\end{theorem}

\begin{proof}
\label{proof_convergence}
Let $\theta^{(t)}$ denote the parameter vector at iteration $t$, and let $\alpha^{(t)}$ be the weight vector selected by APFEx. The combined gradient is:
$$g^{(t)} = \sum_{k=1}^K \alpha_k^{(t)} \nabla L_k(\theta^{(t)})$$

\textbf{Descent property.} By construction of APFEx strategies, the descent direction satisfies:
$$\langle g^{(t)}, d^{(t)} \rangle \geq \|g^{(t)}\|^2 \cos(\phi^{(t)})$$
where $\phi^{(t)}$ is the angle between $g^{(t)}$ and $d^{(t)}$, and $\cos(\phi^{(t)}) \geq \delta > 0$ due to the adaptive strategy selection ensuring alignment.

\textbf{Progress bound.} Using $L$-smoothness of each objective and the weighted combination:
$$L_{\text{combined}}^{(t+1)} \leq L_{\text{combined}}^{(t)} - \eta \delta \|g^{(t)}\|^2 + \frac{L\eta^2}{2}\|d^{(t)}\|^2$$

where $L_{\text{combined}}^{(t)} = \sum_{k=1}^K \alpha_k^{(t)} L_k(\theta^{(t)})$.

\textbf{Telescoping sum.} Summing over $t = 0, \ldots, T-1$ and using the Polyak-Łojasiewicz condition:
$$\sum_{t=0}^{T-1} \|g^{(t)}\|^2 \leq \frac{2(L_{\text{combined}}^{(0)} - L_{\text{combined}}^*)}{\eta\delta - L\eta^2}$$

where $L_{\text{combined}}^*$ is the optimal combined loss value.

\textbf{Convergence rate.} Since $\omega(\theta^{(t)}) \leq \|g^{(t)}\|$ by definition, we have:
$$\min_{t \in \{0, 1, \ldots, T-1\}} \omega^2(\theta^{(t)}) \leq \frac{1}{T} \sum_{t=0}^{T-1} \|g^{(t)}\|^2 \leq \frac{C'}{T}$$

for some constant $C'$. Taking square roots yields the desired $O(1/\sqrt{T})$ convergence rate.
\end{proof}

\end{document}